\documentclass{article} 
\usepackage{iclr2027_conference,times}

\usepackage{amsmath,amsfonts,bm}

\def\eqref#1{equation~\ref{#1}}

\def\1{\bm{1}}

\DeclareMathAlphabet{\mathsfit}{\encodingdefault}{\sfdefault}{m}{sl}
\SetMathAlphabet{\mathsfit}{bold}{\encodingdefault}{\sfdefault}{bx}{n}

\newcommand{\softmax}{\mathrm{softmax}}

\renewcommand{\eqref}[1]{\textup{(\ref{#1})}}

\usepackage{url}
\usepackage{amsmath,amssymb}
\usepackage{amsthm}
\usepackage{booktabs}
\usepackage{array}
\usepackage{algorithm,algorithmic}
\usepackage{multirow}
\usepackage{graphicx}
\usepackage{xcolor}
\usepackage{pifont}
\usepackage{hyperref}
\usepackage{siunitx}
\hypersetup{hidelinks}
\newcommand{\best}[1]{\underline{\textbf{#1}}}
\newcommand{\second}[1]{\textbf{#1}}

\newif\ifdraftmarks
\draftmarkstrue

\newtheorem{proposition}{Proposition}
\newtheorem{theorem}[proposition]{Theorem}
\newtheorem{lemma}[proposition]{Lemma}
\providecommand{\softmax}{\operatorname{softmax}}

\title{PolyStepOR: Learning to Decide \\Without Optimal Decisions}

\author{Viet The Nguyen$^{1,2}$
\quad
Gunther Gust$^{1,2}$
\quad 
An T. Le$^{3,4}$\\[2pt]
$^{1}$University of W\"urzburg, Germany \\
$^{2}$Center for Artificial Intelligence and Data Science (CAIDAS), Germany \\
$^{3}$Center for AI Research, VinUniversity, Vietnam \\
$^{4}$Intelligent Autonomous Systems, TU Darmstadt, Germany
}

\iclrfinalcopy 

\begin{document}
\maketitle
\lhead{Preprint. Under review.}
\begin{abstract}
Decision-focused learning (DFL) trains predictors for downstream decision quality, but often relies on optimal reference decisions that are expensive to obtain. We present PolyStepOR, which trains directly from realized decision costs without pre-computed optima and extends to in-constraint predictions through repair or infeasibility penalties. To handle piecewise-constant losses, PolyStepOR perturbs predictor parameters, evaluates the resulting decisions, and uses optimal transport to favor lower-cost directions, requiring no derivatives. Without task-specific tuning, PolyStepOR performs strongly on classical optimization benchmarks and competitively on predicted-constraint and real-world problems. Theoretically, we characterize decision-preserving perturbations and boundary detection, bound sensitivity to cost errors, and establish stationarity guarantees for a smoothed objective. PolyStepOR thus replaces optimal reference decisions and derivatives with forward evaluations.
\end{abstract}

\section{Introduction}
In many analytical applications, operational decisions are made under uncertainty. In vehicle routing, for instance, routes are determined before travel times are observed~\citep{elmachtoub2022smart}. Decisions in these applications are commonly made through a two-stage approach called \emph{predict-then-optimize}~\citep{elmachtoub2022smart,mandi2024decision}. In the first stage, a machine learning (ML) model predicts the unknown parameters from observed features. The second stage then solves a constrained optimization problem using the predicted parameters. ML models are often trained to minimize prediction errors using statistical losses such as mean-squared error. However, lower prediction error does not necessarily improve decision quality: large errors may leave the decision unchanged, while small errors near a decision boundary can substantially increase its cost. To bridge this gap, \emph{decision-focused learning} (DFL) addresses this mismatch by incorporating downstream decision quality into the training of ML models. This is challenging for many combinatorial problems, as the selected decision is piecewise constant in the predicted parameters, resulting in zero gradients almost everywhere when the loss depends on predictions only through that decision. To obtain useful training signals, existing methods use continuous relaxations~\citep{wilder2019melding}, perturbations or approximate solver derivatives~\citep{berthet2020learning,pogancic2020differentiation}, or surrogate losses~\citep{blondel2020learning,elmachtoub2022smart}. When predicted parameters enter the constraints, decisions may become infeasible under the true parameters, requiring methods that account for feasibility or repair costs~\citep{silvestri2026score, mandi2026feasibility}.

Another challenge is obtaining optimal reference decisions for training, which can be expensive for large instances~\citep{parmentier2025learning,ahmed2024districtnet}. For a fixed feasible set, their costs do not affect the minimizers of expected unnormalized regret because they are independent of the predictor, although the reference decisions themselves can still provide useful training signals. Several DFL methods have proposed learning without optimal targets~\citep{amos2017optnet,gupta2024directional,silvestri2026score}, motivating training from candidate evaluations when optimal decisions are costly to obtain. In this work, we propose PolyStepOR, a DFL framework that applies the forward-only PolyStep optimizer~\citep{le2026polystep} to train predictors directly from realized decision costs. It perturbs predictor parameters along structured directions, solves the resulting problems exactly or heuristically, and updates the predictor using cost-weighted perturbation directions. Evaluations can include repair costs or finite infeasibility penalties for in-constraint predictions. Training requires neither optimal reference decisions nor derivatives of the predictor, solver, or task loss. Finite perturbations can expose decision changes where local gradients vanish, while a restricted search subspace controls the evaluation budget. In summary, our contributions are threefold:
\begin{enumerate}
\item We formulate DFL using forward cost evaluations, including for predicted constraints, and establish when removing optimal reference costs preserves both the training objective's minimizers and PolyStepOR's updates.
\item We characterize decision-preserving perturbations for linear optimization over finite feasible sets, and when a probe detects and crosses an isolated boundary toward a cheaper decision. We also bound the update's sensitivity to evaluation errors and show that the smoothed objective reaches stationarity at rate $O(T^{-1/2})$ after $T$ iterations~\citep{le2026polystep}, up to a residual set by the update bias.
\item We evaluate decision quality on classical optimization benchmarks~\citep{mandi2024decision}, in-constraint predictions~\citep{silvestri2026score}, and extend to practical applications, including districting~\citep{ahmed2024districtnet}, and brass alloy production~\citep{mandi2026feasibility}.
\end{enumerate}
\section{Related work}
\paragraph{Decision-focused learning.}
Existing DFL methods mainly differ in how they obtain useful training signals through optimization problems. Optimization layers differentiate convex problems or continuous relaxations~\citep{amos2017optnet,agrawal2019differentiable,wilder2019melding}, while perturbed optimizers and black-box differentiation smooth or approximate discrete solver maps~\citep{berthet2020learning,pogancic2020differentiation,niepert2021implicit}. Surrogate losses, including SPO+, Fenchel-Young, contrastive, and ranking losses, instead use reference solutions to construct informative gradients~\citep{elmachtoub2022smart,blondel2020learning,mulamba2020contrastive,mandi2022decision}. Reference-free alternatives include perturbation-gradient losses~\citep{gupta2024directional} and LANCER~\citep{zharmagambetov2023landscape}, which learn a differentiable surrogate from decision-cost evaluations. These approaches differ in their supervision and solver requirements but retain gradient-based predictor training.
\paragraph{In-constraint predictions.}
When predicted parameters enter the constraints, decisions may become infeasible under the true parameters. Existing methods account for this through post-hoc correction~\citep{hu2023predict+,hu2023two}, score-function gradient estimation (SFGE)~\citep{silvestri2026score}, or infeasibility penalties combined with supervision from optimal solutions~\citep{mandi2026feasibility}. SFGE avoids differentiating the solver or task loss but differentiates the prediction distribution; PolyStepOR instead perturbs predictor parameters directly. Unlike on-policy DFL, which learns from observed action costs~\citep{benslimane2026onpolicy}, our setting requires evaluating every queried candidate using recorded data or simulation.
\paragraph{Derivative-free DFL.}
Derivative-free DFL has been explored through decision trees~\citep{elmachtoub2020decision} and structured approximations trained without optimal targets~\citep{parmentier2025learning}. For neural predictors, evolution strategies offer a general approach to training from function evaluations~\citep{salimans2017evolution}. PolyStep likewise uses function evaluations, but constructs updates from rotated polytope probes and entropy-regularized cost weights~\citep{le2023sinkhorn,le2026polystep}. We apply this optimizer to DFL and analyze how decision regions and evaluation errors affect its updates. The resulting PolyStepOR framework accommodates in-constraint predictions without optimal training targets or derivatives of the predictor, solver, or task loss.

\section{Problem setup}\label{sec:setup}
Let $\bm{c}\in\mathbb{R}^p$ denote the unknown parameters of a constrained optimization (CO) problem:
\begin{equation}\label{bg:for}
    \bm{x}^\star(\bm{c}) \in \arg\min_{\bm{x}\in\mathcal{X}(\bm{c})} f(\bm{x},\bm{c}),
\end{equation}
where $f:\mathbb{R}^n\times\mathbb{R}^p\rightarrow\mathbb{R}$ is the objective and $\mathcal{X}(\bm{c})\subseteq\mathbb{R}^n$ is the feasible set. The feasible set may contain discrete decisions and depend on $\bm{c}$. Evaluating our method requires neither convexity nor differentiability. For this formulation, we assume an optimum exists, with a fixed tie-breaking rule selecting $\bm{x}^\star(\bm{c})$. Section~\ref{sec:or-interface} extends the setting to heuristic solvers. At decision time, $\bm{c}$ is unknown and predicted from observed correlated features $\bm{z}$:
\begin{equation}
    \hat{\bm{c}}=\phi_{\bm{\theta}}(\bm{z}),
\end{equation}
where $\phi_{\bm{\theta}}$ is an ML model parametrized by $\bm{\theta}$. The predicted parameters are passed to the solver to obtain the decision $\bm{x}^\star(\hat{\bm{c}})$. For a fixed feasible set, its quality is measured by the \emph{regret}:
\begin{equation}\label{eq:regret}
    \texttt{Regret}(\hat{\bm{c}},\bm{c})
    =f(\bm{x}^\star(\hat{\bm{c}}),\bm{c})
    -f(\bm{x}^\star(\bm{c}),\bm{c}),
\end{equation}
which quantifies the additional cost relative to the optimal \emph{reference solution} under the true parameters. Let $(\bm{z},\bm{c}) \sim \mathcal{P}$ denote a random pair of observed features and problem parameters with a joint distribution $\mathcal{P}$. Assuming finite expectations, minimizing expected regret is equivalent to minimizing expected realized decision cost:
\begin{equation}\label{eq:minobj}
\begin{aligned}
    \arg\min_{\bm{\theta}} R(\bm{\theta})
    &=\arg\min_{\bm{\theta}}
    \mathbb{E}_{(\bm{z},\bm{c})\sim\mathcal{P}}
    [\texttt{Regret}(\phi_{\bm{\theta}}(\bm{z}),\bm{c})]\\
    &=\arg\min_{\bm{\theta}}
    \mathbb{E}_{(\bm{z},\bm{c})\sim\mathcal{P}}
    [f(\bm{x}^\star(\phi_{\bm{\theta}}(\bm{z})),\bm{c})].
\end{aligned}
\end{equation}
The reference cost (second term of Eq~\ref{eq:regret}) is omitted because it is independent of $\bm{\theta}$. Proposition~\ref{prop:reference} establishes when this omission also preserves the parameter updates of our algorithm.

In-constraint prediction arises when predicted parameters enter the CO's constraints. For example, stochastic-weight knapsack~\citep{silvestri2026score} selects items with known values and capacity $\bm{b}$ but unknown weights $\bm{c}$, giving $\mathcal{X}(\bm{c})=\{\bm{x}\in\{0,1\}^p:\bm{c}^{\top}\bm{x}\leq\bm{b}\}$. Here, decisions based on the predicted weights may violate the true capacity. A way to tackle this is to perform a \emph{recourse action} to correct the decision to $\bm{x}^{\mathrm{corr}}(\hat{\bm{c}},\bm{c})\in\mathcal{X}(\bm{c})$, incurring a problem-specific penalty $\mathrm{pen}_{\rho}$~\citep{birge1997introduction,hu2023two,hu2023predict+}. The resulting \emph{post-hoc regret} then reads:
\begin{equation}\label{eq:pregret}
    \texttt{PRegret}(\hat{\bm{c}},\bm{c})
    =f(\bm{x}^{\mathrm{corr}}(\hat{\bm{c}},\bm{c}),\bm{c})
    -f(\bm{x}^\star(\bm{c}),\bm{c}) +\mathrm{pen}_{\rho}\bigl(
    \bm{x}^\star(\hat{\bm{c}}),
    \bm{x}^{\mathrm{corr}}(\hat{\bm{c}},\bm{c})\bigr).
\end{equation}
The reference term is again independent of $\bm{\theta}$, so Eq.~\ref{eq:minobj} also applies with the corrected cost and recourse penalty. For our algorithm, \texttt{Regret} and \texttt{PRegret} are used only for test-time evaluation.
\section{Method}\label{sec:method}
PolyStepOR applies PolyStep\footnote{A visualization of the algorithm: \href{https://vietngth.github.io/polystep-visualization/}{https://vietngth.github.io/polystep-visualization/}}~\citep{le2026polystep,le2023sinkhorn} to the prediction--solver pipeline. It perturbs the predictor, evaluates the resulting decisions, and combines the perturbation directions using their realized costs. Finite probes can reveal decision changes even where local derivatives vanish.
\subsection{Learning from realized decision costs}\label{sec:or-interface}
Let $S(\hat{\bm{c}})$ return a decision for predicted parameters $\hat{\bm{c}}$. Our ideal oracle uses a deterministic exact solver or heuristic with fixed tie-breaking; time limits alone need not make a solver deterministic. Appendix~\ref{app:oracle} discusses deviations from this assumption due to evaluation errors. On training data $\mathcal D=\{(\bm{z}_i,\bm{c}_i)\}_{i=1}^N$, the learning objective is:

\begin{equation}\label{eq:objective}
 L(\bm{\theta})=\frac1N\sum_{i=1}^N\ell\bigl(S(\phi_{\bm{\theta}}(\bm{z}_i)),\bm{c}_i\bigr),
\end{equation}
where $\ell(\bm{x},\bm{c})$ is the task loss of decision $\bm{x}$ under the true parameters $\bm{c}$. Evaluating $L$ takes $N$ solver calls per candidate predictor. We instead use a minibatch $\mathcal B$ of $B$ training indices drawn uniformly at random,
\begin{equation}\label{eq:batch-loss}
 \widehat L_{\mathcal B}(\bm{\theta})=\frac1B\sum_{i\in\mathcal B}\ell\bigl(S(\phi_{\bm{\theta}}(\bm{z}_i)),\bm{c}_i\bigr).
\end{equation}
For fixed $\bm\theta$, uniform sampling gives $\mathbb E_{\mathcal B}[\widehat L_{\mathcal B}(\bm\theta)]=L(\bm\theta)$ using $B$ solver calls per candidate. This \emph{cost oracle} requires no gradients or optimal reference decisions. It only requires true instance parameters or a simulation process~\citep{ahmed2024districtnet} to score candidate decisions that may never have been taken in the historical data.
\paragraph{Task loss.} We use $\ell(\bm{x},\bm{c})=f(\bm{x},\bm{c})$ for problems where predictions are only in the objective. If they are in the constraints, the decision $\bm{x}^\star=S(\hat{\bm{c}})$ may be infeasible~\citep{silvestri2026score}. As in Eq.~\eqref{eq:pregret}, we then use the cost of the corrected decision and the recourse penalty:
\begin{equation}\label{eq:recourse-loss}
 \ell(\bm{x}^\star,\bm{c})=f\bigl(\bm{x}^{\mathrm{corr}}(\bm{x}^\star,\bm{c}),\bm{c}\bigr)+\mathrm{pen}_\rho\bigl(\bm{x}^\star,\bm{x}^{\mathrm{corr}}(\bm{x}^\star,\bm{c})\bigr),
\end{equation}
where $\bm{x}^{\mathrm{corr}}(\bm{x}^\star,\bm{c})\in\mathcal X(\bm{c})$ corrects $\bm{x}^\star$ under the true parameters. For an exact solver, this is the post-hoc regret of Eq.~\eqref{eq:pregret} without its reference term. If the solver returns no decision, a fixed finite failure cost or fallback decision is used.
\subsection{Structured parameter perturbations}\label{sec:polystep}
Each iteration evaluates nearby predictors, called \emph{probes}, and favors directions with lower costs. There are 3 components in the algorithm: \emph{search coordinates} control the evaluation budget, \emph{symmetric probes} compare directions, and \emph{cost weights} determine the update. Section~\ref{sec:analysis} separates detecting a decision change from taking a step that improves the loss.

\paragraph{Search coordinates.}
We dentote $\bm{\theta}(\bm{y})=\bm{\theta}_0+A\bm{y}$, with fixed $A\in\mathbb R^{D_\theta\times d}$ and initial search coordinates $\bm{y}=\bm{0}$. Setting $A=I$ searches the full parameter space; fewer columns reduce the number of probes. Partition $d=Pq$ coordinates into $P$ blocks of equal size $q$. Each probe changes one block, with the other blocks held fixed. Distances below are in these search coordinates.

\paragraph{Probes.}
Choose unit vertices $\bm v_1,\ldots,\bm v_V\in\mathbb R^q$ with $\sum_v\bm v_v=0$: either a regular simplex ($V=q+1$) or the orthoplex $\{\pm\bm e_r\}_{r=1}^q$ ($V=2q$). For each block, draw a Haar rotation $R_j\in\mathrm{SO}(q)$. When $q\geq2$, every rotated vertex is uniform on the unit sphere. Probe each direction at radii $\rho_1,\ldots,\rho_K>0$, typically $\rho_k=rk/(K+1)$ for an outer radius $r>0$, and average the costs:
\begin{equation}\label{eq:probes}
\bm{y}_{jvk}=\bm{y}+\rho_k E_jR_j\bm{v}_v,\qquad
C_{jv}=\frac1K\sum_{k=1}^K\widehat L_{\mathcal B}(\bm{\theta}(\bm{y}_{jvk})),
\end{equation}
where $E_j$ embeds a vector in block $j$. All probes use the same minibatch, so their costs compare predictors on the same instances. A shared batch still need not represent the full-dataset preference.

\paragraph{Update.}
Within each block, temperature $\varepsilon>0$ controls how strongly the update favors cheaper directions:
\begin{equation}\label{eq:weights}
w_{jv}=\frac{\exp(-C_{jv}/\varepsilon)}{\sum_{u=1}^V\exp(-C_{ju}/\varepsilon)},
\qquad \bm{D}_j(C)=\sum_{v=1}^Vw_{jv}R_j\bm{v}_v,
\end{equation}
and a step length $s>0$ scales their simultaneous update:
\begin{equation}\label{eq:update}
\bm{y}^+=\bm{y}+s\sum_{j=1}^P E_j\bm{D}_j(C),\qquad \bm{\theta}^+=\bm{\theta}(\bm{y}^+).
\end{equation}

Lower temperature concentrates weight on the cheapest directions, while higher temperature makes the weights more uniform and the resulting displacement smaller. The weights uniquely minimize $\sum_v C_{jv}w_v+\varepsilon\sum_vw_v\log w_v$ on the probability simplex, which gives the one-sided transport interpretation in Appendix~\ref{app:transport}~\citep{le2026polystep}. Because the softmax is nonlinear, even unbiased probe-cost estimates do not generally produce unbiased weights; sampling noise can therefore amplify or weaken the preference for a direction. Since each $\bm D_j$ is a convex combination of unit vectors, Eq.~\eqref{eq:update} implies the step-size bound $|\bm y^+-\bm y|_2\leq s\sqrt P$. Since all blocks update simultaneously, the result generally differs from any single probe.

\paragraph{Momentum.}
The optional heavy-ball update accumulates barycentric displacements. With $D_t=\sum_jE_jD_{j,t}$, it is
\begin{equation}\label{eq:momentum}
 v_{t+1}=\mu_t v_t+s_tD_t,\qquad y_{t+1}=y_t+\eta v_{t+1},\qquad v_0=0,
\end{equation}
where $\eta>0$ scales the velocity and $0\leq\mu_t<1$ controls momentum. We linearly interpolate $\mu_t$ between its initial and final values during training. Setting $\mu_t=0$ and $\eta=1$ recovers Eq.~\eqref{eq:update}. Momentum can carry the predictor across plateaus, but may also oppose a new signal. Hence, $s_t\sqrt P$ bounds only the fresh displacement, not the full momentum step. The implementation resets the velocity when the search basis changes, while the convergence analysis assumes a fixed basis.

\begin{algorithm}[tb]
\caption{One iteration of PolyStepOR}
\label{alg:polystepor}
\begin{algorithmic}[1]
\REQUIRE State $(y_t,v_t)$, with $v_0=0$; reconstruction $\theta(\cdot)$; shared-batch oracle $\widehat L_{\mathcal B}$; block embeddings $E_{1:P}$; centered unit vertices $\bm v_{1:V}$; radii $\rho_{1:K}>0$, temperature $\varepsilon>0$, and step $s_t>0$.
\REQUIRE Momentum $0\leq\mu_t<1$ and multiplier $\eta>0$; use $(\mu_t,\eta)=(0,1)$ without momentum.
\FOR{$j=1,\ldots,P$}
\STATE Draw an independent Haar rotation $R_j\in\mathrm{SO}(q)$.
\STATE $C_{jv}\gets\frac1K\sum_{k=1}^K\widehat L_{\mathcal B}\bigl(\theta(y_t+\rho_kE_jR_j\bm v_v)\bigr)$ for $v=1,\ldots,V$.
\STATE $w_{j:}\gets\softmax(-C_{j:}/\varepsilon)$; \quad $D_{j,t}\gets\sum_{v=1}^Vw_{jv}R_j\bm v_v$.
\ENDFOR
\STATE $D_t\gets\sum_{j=1}^P E_jD_{j,t}$.
\STATE $v_{t+1}\gets\mu_t v_t+s_tD_t$; \quad $y_{t+1}\gets y_t+\eta v_{t+1}$.
\RETURN state $(y_{t+1},v_{t+1})$ and predictor parameters $\theta(y_{t+1})$.
\end{algorithmic}
\end{algorithm}

\paragraph{Evaluation budget.}\label{sec:complexity}
Each iteration evaluates $Q=PVK$ candidate predictors and makes $BQ$ solver calls. Simplex and orthoplex probes use $K(q+1)d/q$ and $2Kd$ evaluations, respectively. Recourse, validation, and restarts add to this count. Appendix~\ref{app:complexity} gives arithmetic and storage bounds; runtime comparisons should also include tuning and baseline reference generation.

\subsection{Invariance to reference costs}\label{sec:reference}
We show that reference subtraction preserves both the objective minimizers (Section~\ref{sec:setup}) and the PolyStepOR updates, because the weights depend only on within-row cost differences.

\begin{proposition}[Reference-cost invariance]\label{prop:reference}
With a shared minibatch across probes and identical temperature and cost scaling, replacing each instance loss by $\ell(\bm{x},\bm{c}_i)-r_i$ for parameter-independent $r_i$ leaves the weights unchanged. Runs with the same initial parameters and velocity, random draws, prescribed schedules, and stopping horizon have identical parameter trajectories, including under Eq.~\eqref{eq:momentum}.
\end{proposition}

This means that computing $r_i=f(\bm{x}^\star(\bm{c}_i),\bm{c}_i)$ gives the same updates from costs and unnormalized regret without computing the references. Their rankings also agree on a fixed validation set. Scaling instance costs by fixed positive factors preserves reference subtraction but changes their relative weighting. Loss-dependent normalization and feasibility-conditional evaluation are handled separately in Appendix~\ref{app:transport}.

\section{When do probes yield a decision signal?}\label{sec:analysis}
Finite probes can reveal cheaper decisions where local derivatives vanish. We ask when they detect a change, how instance costs combine, and whether the resulting motion reduces loss. For the geometric analysis, we assume exact linear optimization over fixed finite feasible sets and decision-only losses, and denote the batch loss in search coordinates by setting $\widehat{\mathcal L}_{\mathcal B}(y)=\widehat L_{\mathcal B}(\theta(y))$. Proofs and extensions are in Appendix~\ref{app:proofs}.

The first result measures how far a probe must reach to change a decision.
\begin{proposition}[Decision margin]\label{prop:margin}
Let $\bm{x}_0$ uniquely minimize $\hat{\bm{c}}(\bm{y})^\top\bm{x}$ over a finite set $\mathcal F$, and suppose $\hat{\bm{c}}(\bm{y})=\bm{b}+M\bm{y}$. The distance to loss of unique optimality in search coordinates is
\begin{equation}\label{eq:subspace-margin}
 r_{\bm{y}}=\min_{\substack{\bm{x}\in\mathcal F\setminus\{\bm{x}_0\}\\M^\top(\bm{x}-\bm{x}_0)\ne\bm{0}}}
 \frac{\hat{\bm{c}}(\bm{y})^\top(\bm{x}-\bm{x}_0)}{\|M^\top(\bm{x}-\bm{x}_0)\|_2},
 \qquad \min\varnothing=+\infty.
\end{equation}
Every perturbation $\|\bm{h}\|_2<r_{\bm{y}}$ preserves the decision. With exact evaluations and a loss depending on parameters only through the returned decision, unchanged decisions at every batch probe imply $\bm{D}_j=\bm{0}$ for every block.
\end{proposition}
The ratio compares the predicted-cost gap between the current decision and an alternative decision with the maximum rate at which a perturbation can close that gap. If the denominator is zero, no perturbation in the search space can close it. The smallest such margin across a finite batch gives the radius below which all current decisions remain unchanged. Under exact decision-only evaluations, probes within this radius therefore yield no new update direction, although momentum from earlier iterations may still move the predictor. These margins are used only for analysis; Appendix~\ref{app:margins} extends the result to individual parameter blocks and nonlinear predictors.

When a nearby boundary separates the current decision from a cheaper one, the probes may detect the cheaper region even if the subsequent update is too small to cross the boundary.
\begin{proposition}[Detection and crossing]\label{prop:boundary}
Fix a batch and a block $j$. Suppose the loss in a neighborhood is
$\widehat{\mathcal L}_{\mathcal B}(\bm{y}+E_j\bm{h})=L_-+J\mathbf1\{\delta+\bm{n}^\top\bm{h}>0\}$,
where $\|\bm{n}\|_2=1$, $J>0$, and either value is allowed at a tie. Use orthoplex probes with exact costs, all inside the neighborhood and none on the boundary. Then
\begin{equation}\label{eq:detection}
 \bm{n}^\top\bm{D}_j<0\quad\Longleftrightarrow\quad
 \rho_{\max}\|R_j^\top\bm{n}\|_\infty>|\delta|,
 \qquad \rho_{\max}=\max_k\rho_k.
\end{equation}
Otherwise $\bm{D}_j=\bm{0}$. If $m=-\bm{n}^\top\bm{D}_j>0$, a block-only update remaining in the neighborhood reaches the lower-cost side when $sm>\delta$.
\end{proposition}
Here, $\delta$ is the signed distance to the boundary, positive on the higher-cost side, while $|R_j^\top \bm{n}|_\infty$ measures probe alignment with the boundary normal. The probe radius determines whether the cheaper side is detected, and the step length whether the update crosses into it. This extends PolyStep's centered planar-jump analysis~\citep{le2026polystep}. In a batch, however, different instances may favor different directions or cross different boundaries, so these signals can interact before the update is formed.

\paragraph{What averaging changes.}
The boundary result above considers a single local boundary, whereas PolyStepOR computes probe costs by averaging over the batch before forming the weights. For orthoplex pairs $\pm u_r$, let $C_{r,\pm}$ be their costs averaged over instances and radii, $\Delta_r=C_{r,+}-C_{r,-}$, and $\pi_r=w_{r,+}+w_{r,-}$. Then
\begin{equation}\label{eq:mean-direction}
 D_j=-\sum_{r=1}^q\pi_r\tanh\!\left(\frac{\Delta_r}{2\varepsilon}\right)u_r.
\end{equation}
Each pair points toward its cheaper side, with strength set by its weight $\pi_r$ and cost difference relative to temperature. Cost differences from different instances can cancel before weighting. A lone changing instance contributes only its jump divided by batch size, weakening the signal at fixed temperature. Proposition~\ref{prop:mean-cost} covers batches and populations, recovering Proposition~\ref{prop:boundary} when changing instances share a boundary and favor the same side.

\paragraph{The role of momentum.}
Stored velocity can carry a moving predictor through a plateau; equal probe costs alone cannot start motion from rest. For a block-only step in the same two-value neighborhood, the crossing condition becomes $\delta+\eta n^\top v_{j,t+1}<0$, where $v_{j,t+1}=E_j^\top v_{t+1}$. Earlier directions can help cross or oppose the current preference (Appendix~\ref{app:momentum-local}).

\paragraph{What repeated steps can guarantee.}
Following PolyStep's smoothing analysis~\citep{le2026polystep}, we study a smoothed objective $F$ obtained by averaging the training loss over nearby search coordinates. A small $\nabla F$ means that nearby perturbations change the average loss only slightly. Under fixed coordinates, temperature, and probe distribution, suitable smoothness, and conditionally unbiased batches, Theorems~\ref{thm:rate} and~\ref{thm:momentum-rate} bound the expected squared gradient by $O(T^{-1/2})$ plus a bias residual from nonlinear weights, block interactions, and cost errors. Proposition~\ref{prop:stability} controls the effect of cost errors relative to temperature~\citep{cohen2023hyperbolic}. Thus, we obtain stationarity guarantees for the smoothed objective, including with momentum, but not acceleration or optimal decisions.

\section{Experiments}\label{sec:experiments}
We evaluate PolyStepOR on classical optimization benchmarks~\citep{mandi2024decision}, in-constraint prediction problems~\citep{silvestri2026score}, and two real-world applications: network districting~\citep{ahmed2024districtnet} and brass alloy production~\citep{mandi2026feasibility}. We build directly on original benchmark implementations, maintaining their predictor architectures and optimization formulations. For the first two benchmarks, we employ fixed hyperparameter configurations for PolyStepOR without task-specific tuning and observe consistently strong performance out of the box. The ML models are trained directly on realized decision costs, and optionally including recourse costs for in-constraint predictions. The optimal reference solutions are used only to \emph{compare test regret against baselines}. Appendix~\ref{app:reproducability} provides full hyperparameter configurations and extended results.

\subsection{Classical optimization benchmarks}\label{sec:exp-classical}
We evaluate PolyStepOR on five classical DFL benchmarks from \citet{mandi2024decision}: two synthetic tasks (Figure~\ref{fig:classical-synthetic}) and three real-world tasks (Figure~\ref{fig:classical-real}), comparing against MSE, SPO+~\citep{elmachtoub2022smart}, DBB~\citep{pogancic2020differentiation}, I-MLE~\citep{niepert2021implicit}, and FY~\citep{blondel2020learning} over 10 seeds. Baseline results are taken from \citet{mandi2024decision}. PolyStepOR is competitive across tasks and levels of model misspecification. On portfolio, it achieves the lowest mean regret across polynomial degrees, with $(0.160\pm0.031)\times10^{-3}$ versus $(0.262\pm0.032)\times10^{-3}$ for SPO+ ($p=0.002$). On knapsack and scheduling, it is statistically indistinguishable from the baseline with the lowest mean regret in all six settings ($p\geq0.19$). Performance is weaker on shortest path, especially at degree~8, where nine of ten runs exhaust the step budget without early stopping. Overall, these results show that realized-cost training can match strong DFL baselines without optimal reference targets, although performance remains problem-dependent.
\begin{figure}[tb]
\centering
\includegraphics[width=0.8\textwidth]{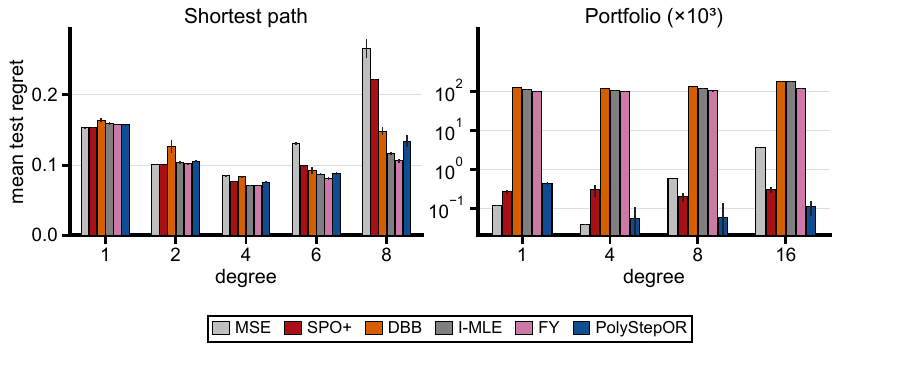}
\caption{Mean test regret on synthetic shortest-path and portfolio tasks as misspecification degree varies; lower is better. The portfolio panel uses a logarithmic axis and the displayed $\times10^3$ scale.}
\label{fig:classical-synthetic}
\end{figure}
\begin{figure}[tb]
\centering
\includegraphics[width=0.9\textwidth]{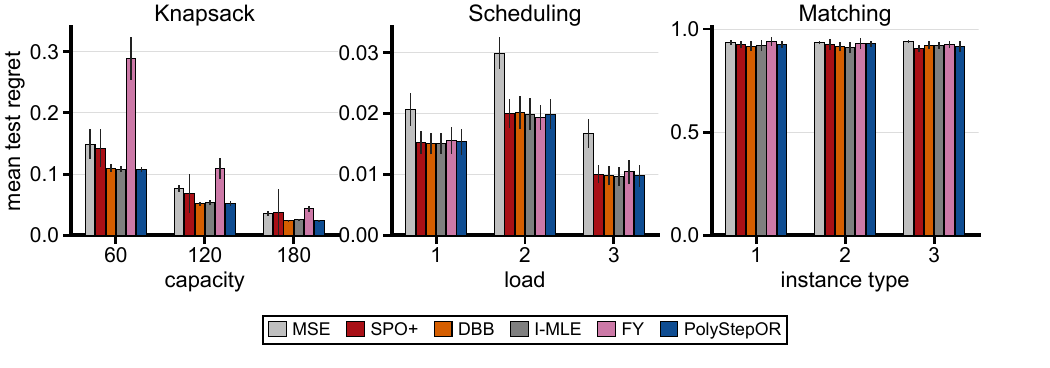}
\caption{Mean test regret for knapsack capacities and scheduling loads on energy-price data~\citep{ifrim2012properties}, and matching instance types on citation networks~\citep{sen2008collective}; lower is better.}
\label{fig:classical-real}
\end{figure}

\subsection{Predicted constraint parameters}\label{sec:exp-inconstraint}
We evaluate three in-constraint problems from~\citet{silvestri2026score}, such as predicting knapsack capacity, item weights, and weighted set multi-cover requirements. We compare against MSE, CombOptNet~\citep{paulus2021comboptnet} where available, and SFGE~\citep{silvestri2026score} over 15 runs per setting. PolyStepOR uses one fixed configuration across all settings without task-specific tuning; full training details are provided in Appendix~\ref{app:additional-results}. Here, PolyStepOR achieves the lowest mean relative post-hoc regret at all three penalty levels for predicted capacity ($p\leq0.002$ versus SFGE). For predicted item weights, SFGE performs better at every penalty level, while PolyStepOR improves over MSE and CombOptNet. On multi-cover (WSMC), PolyStepOR and SFGE do not differ significantly ($p\geq0.25$). Appendix~\ref{app:additional-results} reports prediction error and infeasibility for further analysis.
\begin{table}[tb]
\caption{Mean relative post-hoc regret over 15 runs of 3 in-constraint prediction problems. Bold underlined marks the lowest mean in each row and bold the second lowest.}
\label{tab:inconstraint}

\begin{center}
\resizebox{0.95\textwidth}{!}{%
\begin{tabular}{llcccc}
\toprule
Problem & $\rho$ & MSE & CombOptNet & SFGE & PolyStepOR \\
\midrule
\multirow{3}{*}{KP (capacity)} & 5 & 0.5556 $\pm$ 0.3529 & -- & \second{0.3577 $\pm$ 0.2152} & \best{0.2975 $\pm$ 0.1666} \\
 & 10 & 1.2126 $\pm$ 0.7797 & -- & \second{0.5081 $\pm$ 0.2784} & \best{0.3802 $\pm$ 0.1958} \\
 & 20 & 2.5205 $\pm$ 1.6313 & -- & \second{0.8216 $\pm$ 0.5031} & \best{0.4772 $\pm$ 0.2415} \\
\midrule
\multirow{3}{*}{KP (weight)} & 5 & 0.1676 $\pm$ 0.0357 & 0.1901 $\pm$ 0.0289 & \best{0.1297 $\pm$ 0.0114} & \second{0.1501 $\pm$ 0.0210} \\
 & 10 & 0.3190 $\pm$ 0.0814 & 0.4015 $\pm$ 0.0659 & \best{0.1760 $\pm$ 0.0221} & \second{0.2293 $\pm$ 0.0372} \\
 & 20 & 0.6146 $\pm$ 0.1745 & 0.8233 $\pm$ 0.1401 & \best{0.2205 $\pm$ 0.0276} & \second{0.2884 $\pm$ 0.0558} \\
\midrule
\multirow{3}{*}{WSMC $10\times50$} & 1 & 1.2960 $\pm$ 0.5676 & 5.2044 $\pm$ 2.1887 & \second{1.1373 $\pm$ 0.4368} & \best{1.1205 $\pm$ 0.5002} \\
 & 5 & 5.7115 $\pm$ 2.7809 & 110.7887 $\pm$ 30.2424 & \second{2.5460 $\pm$ 0.9283} & \best{2.5253 $\pm$ 1.0119} \\
 & 10 & 11.2307 $\pm$ 5.5511 & 440.7395 $\pm$ 118.7445 & \best{3.3108 $\pm$ 1.1913} & \second{3.3701 $\pm$ 1.5660} \\
\bottomrule
\end{tabular}}
\end{center}
\end{table}

\subsection{Case Study: Network Districting}
DistrictNet~\citep{ahmed2024districtnet} aims to partition a geographical area into connected districts of bounded size to minimize expected routing cost~\citep{ahmed2024districtnet}. For city $i$, the cost of a feasible districting $\lambda\in\Lambda_i$ is estimated over demand scenarios $\bm{\xi}^{\omega}$ as:
\begin{equation}\label{eq:dn-cost}
C_i(\lambda)
=
\sum_{d\in\lambda}\frac{1}{|\Omega|}
\sum_{\omega\in\Omega}\mathrm{TSP}(d,\bm{\xi}^{\omega}),
\end{equation}
where $\mathrm{TSP}(d,\bm{\xi}^{\omega})$ is the tour length serving district $d$ under scenario $\omega$. Obtaining exact solutions is expensive in this application, which can take roughly 400 CPU-core days for one instance with 60 geographical units and 10 districts~\citep{ahmed2024districtnet}. DistrictNet addresses this by solving a capacitated minimum spanning tree (CMST) surrogate. A neural network first predicts edge weights, and solving the surrogate using the weights then produces a district partition. The model is trained on optimal districtings of small cities, encoded as targets with the FY loss~\citep{blondel2020learning}. We adapt PolyStpPOR on the same network and surrogate but train on the estimated routing costs:
\begin{equation}\label{eq:dn-ours}
\min_{\bm{\theta}}\ \frac{1}{n}\sum_{i=1}^{n}
C_i\bigl(S(\phi_{\bm{\theta}}(\bm{z}_i))\bigr),
\end{equation}
where $S$ maps predicted edge weights to a districting. We use the authors' precomputed district costs to evaluate candidate partitions, without using optimal districtings.

Table~\ref{tab:districtnet-t1} reports results on 35 benchmark problems and the 2,000-unit Ile-de-France instance. Across ten seeds, PolyStepOR is competitive with DistrictNet on the test data. The relative cost is the mean over the 35 problems of each method's cost difference to PolyStepOR divided by the PolyStepOR cost. On the large instance Ile-de-France, the best of ten PolyStepOR runs costs $2156.3$, compared with the published DistrictNet cost of $2205.7$. These results demonstrate competitive districting quality from routing-cost evaluations without optimal training partitions, and can potentially scale up to larger instances for training. We further visualize the real districting solutions in Appendix~\ref{app:districting}.
\begin{table}[t]
\caption{Districting costs on the 35 test problems (left), which consists of 7 real cities at 120 units, 5 target sizes and on the large instance Ile-de-France (right) with 2000 units.}
\label{tab:districtnet-t1}\label{tab:districtnet-t2}
\centering
\begin{tabular}{lccccc}
\toprule
& \multicolumn{3}{c}{35 test problems} & \multicolumn{2}{c}{Ile-de-France} \\
\cmidrule(lr){2-4}\cmidrule(lr){5-6}
Method & Cost & Rel. (\%) & $p$ & Cost & Rel. (\%) \\
\midrule
BD & 594.1 & 9.46 & $\num{1e-8}$ & 2379.0 & 10.33 \\
FIG & 594.3 & 9.60 & $\num{1e-8}$ & 2388.8 & 10.78 \\
PredGNN & 603.0 & 10.95 & $\num{1e-10}$ & 2295.2 & 6.44 \\
AvgTSP & 583.3 & 3.86 & $\num{2e-5}$ & 2262.7 & 4.93 \\
DistrictNet & \second{557.7} & -0.05 & 0.14 & \second{2205.7} & 2.29 \\
\midrule
PolyStepOR & \best{554.0} & 0.00 & -- & \best{2156.3} & 0.00 \\
\bottomrule
\end{tabular}
\end{table}

\subsection{Case Study: Brass alloy production}
We apply PolyStepOR to brass alloy production, where a factory purchases ore from 10 suppliers to meet copper and zinc requirements using predicted metal contents~\citep{mandi2026feasibility}. PolyStepOR trains on purchase costs including recourse penalties for shortages, without optimal reference purchases or an additional feasibility-specific loss. We also evaluate SFGE~\citep{silvestri2026score} using post-hoc regret. ODECE~\citep{mandi2026feasibility} instead combines infeasibility and optimality-preserving losses, using reference solutions and a weight $\alpha$ to control their balance (Appendix~\ref{app:alloy}).

Figure~\ref{fig:alloy-tradeoff} reports infeasibility before correction and regret among feasible purchases. PolyStepOR achieves $55.5\%$ infeasibility and regret of $0.154$, close to MSE at $53.2\%$ and $0.169$, and below SFGE's reported values on both metrics. At similar infeasibility, ODECE with $\alpha=0.4$ has regret of $0.195$; larger $\alpha$ reduces infeasibility at the expense of regret. PolyStepOR therefore achieves competitive results from corrected-cost evaluations alone, close to their strongest baseline (MSE).

\begin{figure}[tb]
\centering
\includegraphics[width=0.65\textwidth]{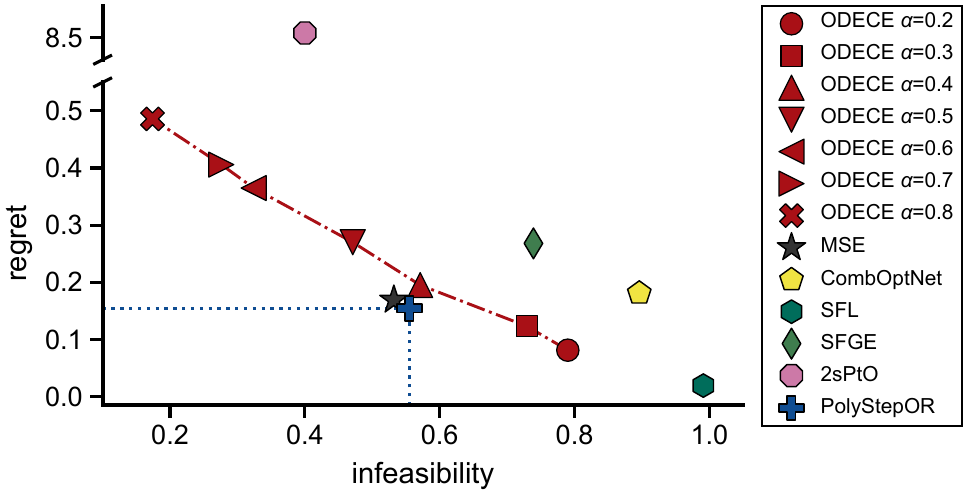}
\caption{Infeasibility ratio and regret of feasible solutions on test instances (adapted from~\citealp{mandi2026feasibility}). PolyStepOR is competitive with ODECE's strongest baseline (MSE).}
\label{fig:alloy-tradeoff}
\end{figure}
\section{Discussion and conclusion}
PolyStepOR trains predictors from realized decision costs without optimal training targets or derivatives of the predictor, solver, or task loss. It accommodates exact and heuristic solvers, including in-constraint predictions evaluated through recourse. Experiments demonstrate competitive decision quality across several benchmarks and practical applications, although performance varies across problems and cost-based training does not ensure first-stage feasibility. Our analysis separates boundary detection from crossing, explains how batch averaging and momentum affect the signal, and bounds the expected squared gradient of a smoothed objective with an explicit bias residual. The bounds cover fixed search coordinates, including scheduled momentum, under sampling and regularity assumptions; they do not guarantee improvement at every step or cover all practical configurations. The main computational cost remains repeated solver calls, making the approach particularly relevant when candidate decisions can be evaluated much more cheaply than optimal reference decisions can be obtained.

Future work includes more efficient parallel evaluations, adaptive probe radii, and analysis of practical optimizer variants. Incorporating an explicit feasibility objective could also provide greater control over the trade-off between decision cost and feasibility.

\clearpage
\subsection*{AI use statement}
LLMs assisted with writing, editorial checks, and literature review. The LLM also helped to implement and debug the experiment code. The authors directed the research, verified all results, and take full responsibility for all methods, proofs, experiments, and claims.
\subsection*{Ethics statement}
This work uses only publicly available benchmark data (synthetic instances, electricity prices, the CORA citation network, public geographic and population data for city districting, and alloy supplier data) and involves no human subjects or personal information.
\subsection*{Reproducibility statement}
To aid complete reproducibility, Appendix~\ref{app:additional-results} provides sets of hyperparameters per experiment conducted in this paper that the audience can set up and reproduce. Our algorithm is available in Julia and Python programming languages. The full code implementation is provided via this link: \href{https://github.com/vietngth/polystep-or}{https://github.com/vietngth/polystep-or}.

\bibliography{iclr2027_conference}
\bibliographystyle{iclr2027_conference}
\clearpage
\appendix
\section{Full analysis and proofs}\label{app:proofs}
We first establish what a single update can learn from its probes, then study repeated updates. The local results explain when a decision change produces a direction and how averaging or momentum changes its effect. The multi-step analysis concerns a smoothed mean loss and states its sampling and regularity assumptions separately. Norms on search coordinates are Euclidean unless stated otherwise.

\subsection{Transport rows and reference costs}\label{app:transport}
Optimal transport assigns mass from sources to targets according to a cost matrix $C\in\mathbb R^{P\times V}$. A plan $T_{jv}\geq0$ records the mass assigned from source $j$ to target $v$, and $\langle C,T\rangle=\sum_{jv}C_{jv}T_{jv}$ is its total cost. Entropic regularization favors distributed assignments when costs are similar~\citep{cuturi2013sinkhorn,peyre2019computational}. Given positive source masses $a_j$ with $\sum_ja_j=1$, the one-sided problem is
\begin{equation}\label{eq:ot}
 \min_{T\geq0,\ T\mathbf1=a}\;
 \langle C,T\rangle+\varepsilon H(T),\qquad
 H(T)=\sum_{jv}T_{jv}(\log T_{jv}-1),\quad \varepsilon>0,
\end{equation}
with $0\log0=0$. Only the source masses are fixed, so each row can be optimized independently. Its unique solution is $T_{jv}=a_jw_{jv}$ with $w_{jv}$ from Eq.~\eqref{eq:weights}. The source masses cancel upon row normalization. In PolyStepOR, $a_j=1/P$ and each target label denotes a block-local direction; it is not a shared point in parameter space.

For a fixed row mass $a_j>0$, write $T_{jv}=a_jw_v$ with $w$ in the probability simplex $\Delta_V$. Up to an additive constant and the positive factor $a_j$, the row objective is $J(w)=\sum_v C_{jv}w_v+\varepsilon\sum_vw_v\log w_v$. Put $Z=\sum_v\exp(-C_{jv}/\varepsilon)$ and $p_v=\exp(-C_{jv}/\varepsilon)/Z$. Direct substitution gives
\[
 J(w)-J(p)=\varepsilon\sum_vw_v\log\frac{w_v}{p_v}\geq0.
\]
The inequality follows from $\log t\leq t-1$: summing $w_v\log(p_v/w_v)\leq p_v-w_v$ over positive $w_v$ gives a nonpositive sum. Since all $p_v>0$, omitting any index with $w_v=0$ makes this bound strict. With full support, equality in $\log t\leq t-1$ requires $w_v=p_v$ for every $v$. Thus $p$ is the unique minimizer.

\begin{proof}[Proof of Proposition~\ref{prop:reference}]
Write $r_{\mathcal B}=B^{-1}\sum_{i\in\mathcal B}r_i$. Replacing $C_{jv}$ by $C_{jv}-r_{\mathcal B}$ multiplies every exponential in a row by $e^{r_{\mathcal B}/\varepsilon}$, which cancels in Eq.~\eqref{eq:weights}.
If two runs have the same parameters and velocity at step $t$, their identical queries therefore give the same $D_t$. Equation~\eqref{eq:momentum} then gives the same velocity and parameters at step $t+1$. Induction from their common initialization proves trajectory invariance. This uses identical prescribed schedules and stopping horizons; an unadjusted rule based on absolute loss values need not be invariant.
\end{proof}

\paragraph{Reference shifts and scaling.}
The same argument allows a separate constant shift for each block. At fixed probes and step length, multiplying every cost and the temperature by the same positive factor also preserves the weights. Temperature therefore has the same units as the cost being weighted.

\paragraph{Why normalization changes the objective.}
Consider two instances with optimal costs $(1,100)$. Suppose two predictors induce cost vectors $(1,110)$ and $(2,100)$, respectively. The first predictor has mean regret $5$, versus $1/2$ for the second, but mean relative regret $1/20$, versus $1/2$. Thus per-instance division by the reference cost reverses their ordering. For fixed positive divisors $d_i$, subtracting $r_i/d_i$ still cancels from each aggregated probe row: normalized regret and normalized cost give the same update. It is the division by $d_i$, not reference subtraction, that changes the objective relative to raw cost. The example uses two feasible decisions per instance, with costs $(1,2)$ and $(100,110)$.

Reference subtraction need not preserve instance-wise batch standardization. Costs $(2,4)$ and references $(1,4)$ give regrets $(1,0)$. The within-batch variances are $1$ and $1/4$, respectively, and the centered vectors even have opposite signs. Proposition~\ref{prop:reference} applies when the same reference is subtracted from each \emph{aggregated probe cost}. It also applies to row centering or row-spread scaling that is itself invariant to common shifts. It does not justify replacing a loss-dependent normalization computed across instances by a different one.

Regret conditional on feasibility poses a different issue: its averaging set depends on the predictor. Even a parameter-independent reference cost can then have a parameter-dependent conditional mean. The common-shift argument applies to a fixed population of instances, including when every instance is scored after repair.

\subsection{Decision margins}\label{app:margins}
For one instance, let $\mathcal F$ be a fixed finite feasible set with at least two decisions and let $x_0$ uniquely minimize $\hat c^\top x$. In a prediction norm $\|\cdot\|$ with dual norm $\|u\|_*=\sup_{\|v\|\leq1}u^\top v$, define
\begin{equation}\label{eq:margin}
 r_c=\min_{x\in\mathcal F\setminus\{x_0\}}
 \frac{\hat c^\top(x-x_0)}{\|x-x_0\|_*}>0.
\end{equation}
The region where $x_0$ is optimal is the polyhedral cone
$\mathcal C_{x_0}=\{c:c^\top(x-x_0)\geq0\text{ for every }x\in\mathcal F\}$.
The exact distance to loss of unique optimality in this norm is $r_c$. If the predictor is locally $H$-Lipschitz, $H>0$, a search perturbation $\|h\|_2<r_c/H$ within that neighborhood preserves the decision. These statements extend Proposition~\ref{prop:margin} to general prediction norms and nonlinear predictors. The affine formula in that proposition gives the exact distance in search coordinates. For block $j$, replace $M^\top$ with $E_j^\top M^\top$ in Eq.~\eqref{eq:subspace-margin} to obtain $r_{y,j}$.

\begin{proof}[Proof of Proposition~\ref{prop:margin} and its extensions]
A point $x_0$ minimizes $c^\top x$ if and only if $c^\top(x-x_0)\geq0$ for all $x\in\mathcal F$, proving the polyhedral description. At the given unique optimum all gaps $m_x=\hat c^\top(x-x_0)$ for $x\ne x_0$ are positive. H\"older's inequality gives
\[
 (\hat c+\Delta c)^\top(x-x_0)
 \geq m_x-\|\Delta c\|\,\|x-x_0\|_*>0
 \quad\text{when }\|\Delta c\|<r_c.
\]
Local Lipschitz continuity bounds $\|\Delta c\|$ by $H\|h\|_2$ whenever the perturbation remains in the stated neighborhood. If $H=0$, the prediction is constant on that neighborhood and the same conclusion holds there without dividing by $H$.

The prediction-space radius is also exact. Choose a decision $x$ attaining the minimum in Eq.~\eqref{eq:margin}. The unit ball is compact, so there is a unit vector $v$ with $v^\top(x-x_0)=\|x-x_0\|_*$. At $\Delta c=-r_cv$, this gap is zero and all other gaps are nonnegative by the same inequality. Hence $x_0$ loses unique optimality at distance $r_c$.

For the affine map, let $a_x=M^\top(x-x_0)$. The perturbed gap is $m_x+a_x^\top h$. If $a_x=0$, it stays strictly positive for every $h$. Otherwise, the nearest point on its zero hyperplane is
\[
 h_x=-\frac{m_x}{\|a_x\|_2^2}a_x,
 \qquad \|h_x\|_2=\frac{m_x}{\|a_x\|_2}.
\]
Every $h$ of norm less than the minimum of these distances has all gaps positive. At a minimizing $h_x$, that gap is zero and every other gap is nonnegative by the same norm bound. Thus unique optimality is lost at exactly $r_y$. If no $a_x$ is nonzero, all gaps stay positive everywhere and the distance is infinite. At the distance itself, the selected decision depends on tie-breaking. If the prediction map is affine only locally, the distance formula is valid only while the relevant ball and witness remain in that affine region.

Replacing the affine map $M$ with $ME_j$ gives the block-specific radius by the same argument. Finally, unchanged decisions and a decision-only task loss make every probe in a row have the same cost. Its softmax weights equal $1/V$, so $D_j=V^{-1}R_j\sum_vv_v=0$. The argument assumes exact evaluations, or common evaluation noise that preserves these equalities; independent noisy evaluations can induce motion even on a plateau.
\end{proof}

The Lipschitz bound gives a sufficient radius, while the affine formula gives the exact distance to a tie in the search coordinates. For a finite batch, let $r_{i,j}$ be instance $i$'s block-$j$ radius. If $\rho_{\max}<\min_{i\in\mathcal B}r_{i,j}$, every instance keeps its decision at every block-$j$ probe and $D_j=0$. The batch mean may stay constant beyond this radius because distinct decisions can have equal costs or different instances' cost changes can cancel. Zero $D_j$ describes the fresh signal; momentum may still move the parameters.

\paragraph{A changed cost need not give a fresh direction.}
For a regular unit simplex with $V=q+1$, PolyStep's identity~\citep{le2026polystep} gives
\[
 \left\|\sum_vw_vR v_v\right\|_2^2
 =\frac{V}{V-1}\sum_v(w_v-1/V)^2.
\]
To check the identity, use $v_v^\top v_u=-1/(V-1)$ for $v\ne u$ and $\sum_vw_v=1$:
\[
 \left\|\sum_vw_vRv_v\right\|_2^2
 =\sum_vw_v^2-\frac{1-\sum_vw_v^2}{V-1}
 =\frac{V}{V-1}\left(\sum_vw_v^2-\frac1V\right).
\]
Thus the simplex direction is zero exactly when the finite-temperature cost row is constant. For an orthoplex, the row $(0,0,1,1)$ on $(e_1,-e_1,e_2,-e_2)$ has unequal costs but zero direction, because each antipodal pair has equal weights. A decision change, a cost change, a fresh direction, parameter motion, and a loss decrease are distinct events. These directions live in search coordinates; a nonzero displacement can also vanish under reconstruction if it lies in the null space of $A$.

\subsection{Updates at decision boundaries}\label{app:boundary}
To determine whether an update approaches a boundary, only its component along the normal matters. Orient each antipodal pair as $\pm u_r$ with $a_r=n^\top u_r\geq0$. Let $w_{r,+},w_{r,-}$ be the weights from Eq.~\eqref{eq:weights}, $\pi_r=w_{r,+}+w_{r,-}$ their sum, and
$\alpha_r=K^{-1}\#\{k:\rho_k a_r>|\delta|\}$ the fraction of radii for which the pair straddles the boundary. Then
\begin{equation}\label{eq:boundary}
 n^\top D_j=-\sum_{r=1}^q\pi_r a_r
                   \tanh\!\left(\frac{J\alpha_r}{2\varepsilon}\right)\leq0.
\end{equation}
Each term combines the pair's weight, its alignment with the normal, and its cost contrast. At $\delta=0$, the no-tie assumption gives $\alpha_r=1$ and $\pi_r=1/q$, recovering PolyStep's planar-jump formula~\citep{le2026polystep}. Away from the boundary, only straddling radii contribute to the cost difference within a pair.
An orientation-independent sufficient condition for detection is $\rho_{\max}>\sqrt q\,|\delta|$, provided all probes remain in the two-decision neighborhood.

\begin{proof}[Proof of Proposition~\ref{prop:boundary}]
Orient $u_r=\pm R_je_r$ so that $a_r=n^\top u_r\geq0$, choosing either sign if $a_r=0$. Let $C_{r,+}$ and $C_{r,-}$ denote the mean probe costs on $u_r$ and $-u_r$. At radius $\rho_k$, the pair is on opposite sides exactly when $\rho_k a_r>|\delta|$; otherwise the two costs agree. No equality case occurs because no probe lies on the boundary. Since the positive member always has at least the negative member's cost,
\[
 C_{r,+}-C_{r,-}=J\alpha_r.
\]
Let $Z=\sum_r[\exp(-C_{r,+}/\varepsilon)+\exp(-C_{r,-}/\varepsilon)]$. The pair's contribution to the normal component is
\[
 a_r\frac{e^{-C_{r,+}/\varepsilon}-e^{-C_{r,-}/\varepsilon}}{Z}
 =-\pi_r a_r\frac{1-e^{-J\alpha_r/\varepsilon}}{1+e^{-J\alpha_r/\varepsilon}}
 =-\pi_r a_r\tanh\!\left(\frac{J\alpha_r}{2\varepsilon}\right).
\]
Summing proves Eq.~\eqref{eq:boundary}. Every $\pi_r$ is positive. A summand is strictly negative exactly when $\alpha_r>0$, which also implies $a_r>0$. Thus strict drift is equivalent to $\rho_{\max}\max_ra_r>|\delta|$. If no pair straddles, its two weights agree, so every antipodal contribution cancels and $D_j=0$.

The signed distance after a block-only update is $\delta+s n^\top D_j=\delta-sm$. If this is negative and the update stays in the stated neighborhood, its loss is $L_-$. For $\delta>0$, this decreases the loss by $J$; for $\delta<0$, it preserves the lower loss. At $\delta=0$, strict decrease depends on the incumbent tie-breaking rule. Every pair then straddles, giving $\alpha_r=1$ and $\pi_r=1/q$. Haar rotations are non-tangent with probability one for $q\geq2$.

Finally, $\sum_ra_r^2=1$ implies $\max_ra_r\geq1/\sqrt q$, proving the sufficient radius. This worst-orientation threshold is sharp: an orthogonal orientation can give $a_r=1/\sqrt q$ for all $r$. The lower threshold $\rho_{\max}\leq|\delta|$ prevents strict straddling for every orientation, apart from excluded ties. Intermediate radii depend on alignment with the boundary normal.
\end{proof}

\paragraph{Detection without an immediate decision change.}
Take $n=(3/5,4/5)$, $R=I$, radii $(1/2,1)$, $\delta=9/20$, $J=2\log3$, and $\varepsilon=1$. Only the larger radius straddles the boundary in each pair, so $\alpha_1=\alpha_2=1/2$. After subtracting $L_-$, each positive--negative pair has costs $(2\log3,\log3)$ and unnormalized weights $(1/9,1/3)$. Normalizing all four weights gives $(1/8,3/8)$ in each pair. Therefore $D=(-1/4,-1/4)$ and $m=7/20$. A unit step leaves the predictor on the higher-cost side because $\delta-m=1/10$. A step $s>9/7$ crosses, provided it stays in the two-decision neighborhood.

The boundary model arises when two distinct feasible decisions exchange optimality under an affine predictor, other decisions are strictly worse throughout the neighborhood, and their true task costs differ. Orienting the normal toward the higher-cost decision gives the stated loss. A batch also has this form if just one instance changes, or if several instances share a boundary with a nonzero aggregate jump. The next result handles different boundaries and conflicting costs.

\paragraph{A counterexample with three decisions.}
An update can increase the task loss when its probes cross more than one decision boundary. Let $\mathcal F=\{e_1,e_2,e_3\}$, minimize predicted coefficients, and evaluate the selected decision under true costs. Take
\[
 \hat c(y)=(0,-y_1+4/5,-y_1/2+31/100),\qquad c=(1,0,2),\qquad y=(0,0).
\]
Use orthoplex directions $(e_1,-e_1,e_2,-e_2)$, $R=I$, $K=1$, and probe radius $1$. At $y_1=0$, $e_1$ is uniquely optimal and costs $1$; at the positive unit probe $e_2$ is uniquely optimal and costs $0$; the other three probes select $e_1$ and cost $1$. With $\varepsilon=1/\log3$, the weights are $(1/2,1/6,1/6,1/6)$ and $D=(1/3,0)$. Taking $s=21/10$ gives $y_1^+=7/10$, where $e_3$ is uniquely optimal and costs $2$. Thus all probe costs are at most the incumbent cost, yet the barycentric step increases the true cost. The lower-cost decision region at $y_1=1$ is separated from the incumbent by a worse region.

\paragraph{Why the crossing guarantee is block-only.}
Two individually improving block moves can interact badly. Take $P=q=2$, $\mathcal F=\{e_1,e_2,e_3,e_4\}$, predicted costs $(0,1/4-y_1,1/4-y_3,1/2-y_1-y_3)$, and true costs $(1,0,0,2)$. At $y=0$ the loss is $1$. Unit orthoplex probes in either block have costs $(0,1,1,1)$. With $\varepsilon=1/\log3$, each block direction is $(1/3,0)$. Either block-only unit step reaches cost $0$, but their simultaneous step reaches $y=(1/3,0,1/3,0)$ and selects $e_4$, with cost $2$. The single-boundary argument for a fixed block therefore cannot be applied to all blocks at once.

\subsection{How instance costs combine}\label{app:mean-cost}
An instance may favor a direction that is poor for the rest of the batch. The relevant quantity is therefore its contribution to the mean cost difference across each probe pair. Let $\mathcal L_i$ denote instance $i$'s loss in search coordinates. The next result expresses the update in terms of these differences, then gives a sufficient condition for the instances to agree.

\begin{proposition}[Mean-cost signal]\label{prop:mean-cost}
Fix a block, orthonormal directions $u_r=R_je_r$, positive radii, and $\varepsilon>0$. Let $\mathbb E_i$ denote either a finite weighted average over instances or a population expectation. Assume all queried losses are measurable and integrable, and define
\[
 C_{r,\pm}=\mathbb E_i\frac1K\sum_k\mathcal L_i(y\pm\rho_kE_ju_r),
 \qquad \Delta_r=C_{r,+}-C_{r,-}.
\]
With $\pi_r=w_{r,+}+w_{r,-}>0$, Eq.~\eqref{eq:mean-direction} holds. In particular, $D_j=0$ if and only if every $\Delta_r=0$.

Suppose, additionally, that each instance has a local two-value loss
$\mathcal L_i(y+E_jh)=a_i+J_i\mathbf1\{\delta_i+n_i^\top h>0\}$,
where the instance parameters are measurable, $J_i\geq0$, $\|n_i\|_2=1$, all probes lie in the stated neighborhoods, and probe ties occur with zero instance probability. If $\mathbb E_iJ_i<\infty$, then
\begin{equation}\label{eq:mean-jump}
 \Delta_r=\mathbb E_i\big[J_i\operatorname{sgn}(n_i^\top u_r)\alpha_{ir}\big],
 \qquad \alpha_{ir}=\frac1K\sum_k
 \mathbf1\{\rho_k|n_i^\top u_r|>|\delta_i|\}.
\end{equation}
If $(n_i^\top u_r)\Delta_r\geq0$ for every $r$ and almost every instance with $J_i>0$, then $n_i^\top D_j\leq0$ for those instances. A block-only step along $D_j$ cannot increase their mean loss if its segment remains in the neighborhoods and its endpoints have zero probability of ties.
\end{proposition}
\begin{proof}
Within each antipodal pair, divide the difference of its weights by their sum:
\[
 w_{r,+}-w_{r,-}
 =\pi_r\frac{e^{-C_{r,+}/\varepsilon}-e^{-C_{r,-}/\varepsilon}}
                 {e^{-C_{r,+}/\varepsilon}+e^{-C_{r,-}/\varepsilon}}
 =-\pi_r\tanh\!\left(\frac{\Delta_r}{2\varepsilon}\right).
\]
Multiplying by $u_r$ and summing proves Eq.~\eqref{eq:mean-direction}. Since the $u_r$ form an orthonormal basis, the sum vanishes exactly when each coefficient vanishes, equivalently when every $\Delta_r=0$.

For an instance and radius, the difference of the two indicators is
$\operatorname{sgn}(n_i^\top u_r)\mathbf1\{\rho_k|n_i^\top u_r|>|\delta_i|\}$:
it is zero when both probes are on the same side and has the stated sign when they straddle. Averaging proves Eq.~\eqref{eq:mean-jump}; integrability justifies exchanging the finite sum and expectation. Finally,
\[
 n_i^\top D_j=-\sum_r\pi_r(n_i^\top u_r)
                      \tanh\!\left(\frac{\Delta_r}{2\varepsilon}\right)\leq0
\]
under the sign condition. Along $h=sD_j$, $s\geq0$, each signed distance is nonincreasing. Hence each indicator, and then its mean cost, cannot increase between the specified endpoints.
\end{proof}

\paragraph{Agreement and cancellation.}
The sign condition in Proposition~\ref{prop:mean-cost} requires each pair's aggregate preference to agree with every affected instance's orientation. It guarantees that almost every instance's cost is nonincreasing. A decrease in the mean can also occur through tradeoffs between instances, so this sufficient condition is stronger than mean-cost improvement alone.

When all $n_i=n$, orient each $u_r$ so $n^\top u_r\geq0$. Then $\Delta_r=\mathbb E_i[J_i\alpha_{ir}]\geq0$, and $n^\top D_j<0$ exactly when $\Delta_r>0$ for some $r$. If the offsets are also equal, Proposition~\ref{prop:boundary} follows with jump $J=\mathbb E_iJ_i>0$. In a batch of size $B$, a lone changing instance contributes $J_i/B$ to the aggregate jump. Its weight contrast is therefore governed by $J_i/(B\varepsilon)$, which can be small even for a substantial individual cost change.

For conflicting instances, take losses $\mathcal L_1(h)=\mathbf1\{n^\top h>0\}$ and $\mathcal L_2(h)=a\mathbf1\{n^\top h<0\}$, away from ties. Their equally weighted mean is $a/2+(1-a)\mathbf1\{n^\top h>0\}/2$. At $a=1$ all mean probe costs agree even though both decisions change; at $a=3$ the mean favors the opposite side to instance 1. These are finite linear decision problems: predicted costs $(-n^\top h,n^\top h)$ choose $e_1$ on the positive side, and true costs $(1,0)$ and $(0,a)$ give the stated losses. A minibatch containing only instance 1 can thus point against the full-data direction. Even without cancellation, the three-decision example above rules out a general loss-decrease claim.

\paragraph{Population means need not have plateaus.}
Finite-batch margins do not imply a positive population margin. For $Z\sim\operatorname{Unif}[-1,1]$, predicted costs $(Z-y,0)$ and true costs $(1,0)$ give loss $\mathbf1\{y>Z\}$ away from ties. At $y=0$, each instance has positive margin $|Z|$ almost surely, yet the population mean is $(1+y)/2$ for $|y|<1$. A positive essential lower bound on the margins suffices to preserve almost every decision throughout a common ball; positive individual margins alone do not. For bounded measurable losses, population averaging also commutes with the smooth convolution used below. This yields a smoothed population target when batches are sampled freshly from that population; empirical training alone gives no population or generalization guarantee.

\subsection{Momentum at a boundary or on a plateau}\label{app:momentum-local}
Momentum reuses directions computed at earlier points. This can sustain motion when the current probes give equal costs, but those earlier directions may no longer favor the current decision. Equation~\eqref{eq:momentum} stores a displacement, with no factor $1-\mu_t$ on the new input. Unrolling it shows exactly how much of each past displacement remains:
\begin{equation}\label{eq:velocity-sum}
 v_{t+1}=\sum_{k=0}^t s_kD_k\prod_{l=k+1}^t\mu_l,
 \qquad
 \|v_{t+1}\|_2\leq\sqrt P\sum_{k=0}^t s_k\prod_{l=k+1}^t\mu_l,
\end{equation}
where an empty product is one. Both statements follow by substitution and $\|D_k\|_2\leq\sqrt P$. If $s_k\leq s_{\max}$ and $\mu_k\leq\bar\mu<1$, the geometric series bounds the actual step by $\eta s_{\max}\sqrt P/(1-\bar\mu)$. Thus the scale of the current step alone does not bound motion with momentum.

For the isolated boundary in Proposition~\ref{prop:boundary}, write $v_{j,t}=E_j^\top v_t$, $m_t=-n^\top D_{j,t}$, and $a_t=-n^\top v_{j,t}$. A block-only momentum step has new signed distance
\begin{equation}\label{eq:momentum-crossing}
 \delta_{t+1}=\delta_t-\eta(\mu_ta_t+s_tm_t).
\end{equation}
It reaches the lower side exactly when $\eta(\mu_ta_t+s_tm_t)>\delta_t$, provided the update remains in the two-value neighborhood. A detected improvement gives $m_t>0$, but a sufficiently negative $a_t$ can reverse the move. For example, $\mu_t=1/2$, $a_t=-1$, $s_t=1$, and $m_t=1/4$ give a displacement toward the higher-cost side. Conversely, stored velocity toward the lower-cost side can carry the update across even if $m_t=0$.

If $D_t,\ldots,D_{t+H-1}=0$ and $\mu_t=\mu\in[0,1)$ throughout these steps, direct summation gives
\begin{equation}\label{eq:plateau-coasting}
 y_{t+H}-y_t=\eta\frac{\mu(1-\mu^H)}{1-\mu}v_t.
\end{equation}
The distance traveled on a plateau is thus limited by the incoming velocity. With $v_t=0$ the trajectory remains fixed. Resetting velocity when the basis changes ends this accumulation. In particular, a reset after every step reduces the next step to $y_{t+1}=y_t+\eta s_tD_t$ in its current coordinates. The finite-horizon theorem below assumes fixed coordinates; refreshed bases need a separate analysis even when velocities are reset.

\subsection{Cost errors and sampling}\label{app:oracle}
We ask how accurately candidate costs must be evaluated to preserve a directional signal. A common error across a row cancels, so the relevant error is measured after removing such an offset. The following bound is the finite softmax form of the sharp total-variation inequality of \citet{cohen2023hyperbolic}; we give an elementary proof below.
\begin{proposition}[Cost-error stability]\label{prop:stability}
For $\varepsilon>0$, $\beta\geq0$, and unit directions $u_v$, set $D(C)=\sum_v\softmax(-C/\varepsilon)_vu_v$. If $\|\widehat C-C-b\mathbf1\|_\infty\leq\beta$ for some scalar $b$, then
\begin{equation}\label{eq:stability}
 \|D(\widehat C)-D(C)\|_2\leq 2\tanh\!\left(\frac{\beta}{2\varepsilon}\right)\leq\min\{2,\beta/\varepsilon\}.
\end{equation}
If $\|n\|_2=1$ and an exact direction has $n^\top D(C)=-m<0$, its noisy counterpart still points toward the lower side whenever $2\tanh(\beta/(2\varepsilon))<m$; the simpler condition $\beta<\varepsilon m$ suffices. The first bound is attained by two antipodal directions and suitable two-entry cost rows.
\end{proposition}
Lower temperature strengthens a reliable cost preference and also amplifies errors in that preference. The ratio $\beta/\varepsilon$ quantifies this tradeoff. Here $\beta$ concerns realized task costs. A solver's gap under predicted coefficients measures a different quantity, as the counterexample below shows.

\begin{proof}[Proof of Proposition~\ref{prop:stability}]
Put $e=\widehat C-C-b\mathbf1$, $p=\softmax(-C/\varepsilon)$, and $\widehat p=\softmax(-\widehat C/\varepsilon)$. The common factor from $b$ cancels, so with $X_v=\exp(-e_v/\varepsilon)$ and $m=\sum_vp_vX_v$ we have $\widehat p_v=p_vX_v/m$. Thus the problem is to bound the change in a probability vector after multiplying each entry by a factor in a fixed interval.

If $\beta=0$, the weights agree. Otherwise let $a=\exp(-\beta/\varepsilon)$ and $b_+=\exp(\beta/\varepsilon)$, so $a\leq X_v\leq b_+$ and $a\leq m\leq b_+$. For $x\in[a,b_+]$, convexity gives the explicit chord bound
\[
 |x-m|\leq\frac{b_+-x}{b_+-a}(m-a)
              +\frac{x-a}{b_+-a}(b_+-m).
\]
Substitute $x=X_v$, multiply by $p_v$, and sum. Since $\sum_vp_vX_v=m$, this yields
\[
 \sum_vp_v|X_v-m|\leq\frac{2(b_+-m)(m-a)}{b_+-a}.
\]
Dividing by $m$ and using $m+ab_+/m\geq2\sqrt{ab_+}$ yields
\[
 \|\widehat p-p\|_1
 \leq\frac{2(a+b_+-m-ab_+/m)}{b_+-a}
 \leq 2\frac{\sqrt{b_+}-\sqrt a}{\sqrt{b_+}+\sqrt a}
 =2\tanh\!\left(\frac{\beta}{2\varepsilon}\right).
\]
The triangle inequality and $\|u_v\|_2=1$ bound the direction error by this $\ell_1$ distance. Since $\tanh t\leq\min\{1,t\}$ for $t\geq0$, the linear and diameter bounds follow. Taking the inner product with a unit normal proves the drift claim. Sharpness follows with $u_1=-u_2$ unit, $C=(-\beta/2,\beta/2)$, and $\widehat C=(\beta/2,-\beta/2)$: the two directions differ by exactly $2\tanh(\beta/(2\varepsilon))$. This is the finite softmax case of \citet{cohen2023hyperbolic}.
\end{proof}

\paragraph{Minibatches versus exact mean costs.}
Condition on the current parameters, rotations, and radii, chosen independently of the current batch. Let $C^\star$ contain the exact mean costs, either over the training set or over a population, and let $\widehat C$ use $B$ i.i.d. instances shared across all probes. Suppose each queried instance loss lies in a common interval of width $W$. Then, for $0<\alpha<1$, with conditional probability at least $1-\alpha$,
\begin{equation}\label{eq:batch-concentration}
 \max_{j,v}|\widehat C_{jv}-C^\star_{jv}|
 \leq\beta_B:=W\sqrt{\frac{\log(2PV/\alpha)}{2B}},
 \qquad
 \|D_j(\widehat C)-D_j(C^\star)\|_2
 \leq2\tanh\!\left(\frac{\beta_B}{2\varepsilon}\right).
\end{equation}
To see this, average the $K$ radii within each instance first. Each resulting variable still lies in an interval of width $W$, and the $B$ variables for a fixed entry are independent. Hoeffding's inequality~\citep{hoeffding1963probability} gives entrywise failure probability at most $2\exp(-2B\beta_B^2/W^2)$; a union bound over the $PV$ entries proves the first bound. Proposition~\ref{prop:stability} proves the second. Shared instances create dependence between entries, which the union bound does not require us to remove. For $W=0$, all errors are zero.

The conditional expected direction error in each block is consequently at most
\begin{equation}\label{eq:batch-expected-error}
 \chi_B:=\min\left\{2,\,2\tanh\!\left(\frac{\beta_B}{2\varepsilon}\right)+2\alpha\right\}.
\end{equation}
On the success event use Eq.~\eqref{eq:batch-concentration}; on failure the distance between two convex combinations of unit vectors is at most two. A full-mean normal drift $-m<0$ is therefore preserved on the success event whenever $2\tanh(\beta_B/(2\varepsilon))<m$. This quantifies when a batch represents the mean signal. It does not imply unbiased softmax weights, or that noise always shortens the step.

For a concrete comparison, take directions $\pm1$, $B=1$, and $\varepsilon=1/\log3$, so $D(C)=-\tanh((C_+-C_-)/(2\varepsilon))$. If the instance cost row is equally likely to be $(0,0)$ or $(2,0)$, then $D(\mathbb EC)=-1/2$ but $\mathbb ED(C)=-2/5$: sampling weakens the preference. If instead the row is $(0,8)$ with probability $1/10$ and $(2,0)$ with probability $9/10$, then $D(\mathbb EC)=-1/2$ again, whereas $\mathbb ED(C)=\tanh(4\log3)/10-18/25<-31/50$. Sampling now strengthens it. Both examples have unbiased probe costs.

\paragraph{Simulation estimates.}
Suppose each of $Q$ probe evaluations averages $M$ independent, conditionally identically distributed simulation costs in an interval of width $W$, with means equal to their target losses. Conditional on the candidate queries, Hoeffding's inequality and a union bound give simultaneous error at most
\begin{equation}\label{eq:hoeffding}
 \beta=W\sqrt{\frac{\log(2Q/\alpha)}{2M}}
\end{equation}
with probability at least $1-\alpha$, for $0<\alpha<1$. Independence across candidates is unnecessary; independence within each estimate suffices. Averaging these evaluations along a direction does not enlarge their maximum error. The bound assumes that the simulation estimates are unbiased for the target costs.

\paragraph{A solver-gap counterexample.}
For $\mathcal F=\{e_1,e_2\}$, let the predicted costs be $(0,\eta)$ with $\eta>0$, and the true costs $(0,M)$ with any $M>0$. Returning $e_2$ instead of the exact predicted optimum $e_1$ incurs predicted additive gap $\eta$ but changes the realized cost by $M$. Since $M$ is arbitrary, the predicted optimality gap alone does not bound downstream cost error. Such a bound requires additional assumptions linking predicted and true costs.

\subsection{Smoothing and the explicit bias}\label{app:smoothing}
The boundary results explain individual updates, but their isolation assumptions need not hold along a training run. For repeated steps, we measure progress on a spatial average of the mean loss. This gives a notion of stationarity even when the unsmoothed decision loss is discontinuous. Write $\mathcal L(y)=L(\theta(y))$ for the empirical objective. The same argument applies to a bounded population mean with fresh population samples. Throughout, $A$ is fixed and gradients are with respect to $y$, so the guarantee concerns the chosen search space.

\paragraph{The target and bound.}\label{sec:convergence}
The smoothing scale determines which nearby decision changes the target averages together. Define
\begin{equation}\label{eq:joint-smoothing}
 F(y)=\mathbb E_{\xi}[\mathcal L(y+\xi)],\qquad \xi=(\xi_1,\ldots,\xi_P),
\end{equation}
where each block $\xi_j$ is drawn independently from a mixture of uniform distributions on balls centered at zero. For fixed probe radii, the ball of radius $\rho_k$ has mixture weight $\rho_k/\sum_u\rho_u$. This radius weighting follows from the relation between spherical probes and derivatives of ball averages. We derive this relation and define the random-radius case below.

The analysis compares the actual softmax direction with its linear approximation; the algorithm continues to use the exact weights. For each block, averaging the linear term over rotations gives a scaled negative gradient of the loss smoothed in that block. Nonlinear weighting, block interactions, and cost errors account for the remaining bias. Under the sampling and regularity conditions below, the complete direction $D_t=(D_{1,t},\ldots,D_{P,t})$ satisfies
\begin{equation}\label{eq:bias-model}
 \mathbb E[D_t\mid\mathcal H_t]=-\gamma\nabla F(y_t)+e_t,
 \qquad \|e_t\|_2\leq b,\qquad
 \gamma=\frac{\bar\rho}{q\varepsilon},
\end{equation}
where $\mathcal H_t$ is the history before iteration $t$ and $\bar\rho$ is the mean probe radius. For deterministic radii, $\bar\rho=K^{-1}\sum_k\rho_k$. Let $G_j$ smooth only block $j$. If $\sigma_j^2$ bounds the conditional second moment of the ideal cost-row spread, $\kappa_j$ bounds $\|\nabla_jG_j-\nabla_jF\|_2$, and $\beta_j$ bounds the cost error after a row shift, then
\begin{equation}\label{eq:explicit-bias}
 b^2=\sum_{j=1}^P\left(
 \underbrace{\frac{\sigma_j^2}{4\varepsilon^2}}_{\text{softmax}}
 +\underbrace{\gamma\kappa_j}_{\text{block discrepancy}}
 +\underbrace{2\tanh(\beta_j/(2\varepsilon))}_{\text{cost error}}
 \right)^2.
\end{equation}
The block term appears because a probe changes one block while $F$ averages all blocks. The other terms bound the departure from linear cost weights and the error in evaluated costs. When $\|\nabla F(y_t)\|_2>b/\gamma$, the gradient contribution dominates these errors and the expected direction is a descent direction for $F$. Smooth descent then gives the following bound~\citep{ghadimi2013stochastic}.

\begin{theorem}[Finite-horizon bound with bias]\label{thm:rate}
Assume a deterministic initialization $y_0$, $F\geq F_{\inf}$, $F$ is differentiable with $L_F$-Lipschitz gradient on a convex set containing the iterates and update segments, and Eq.~\eqref{eq:bias-model} holds almost surely with fixed $\gamma>0$ and $b\geq0$. Let $D_t$ be the concatenated block directions from Eq.~\eqref{eq:weights}, so $\|D_t\|_2\leq\sqrt P$.
Run $y_{t+1}=y_t+sD_t$ for $T$ steps with $s=a/\sqrt T$, $a>0$. For an independently uniform index $\tau\in\{0,\ldots,T-1\}$,
\begin{equation}\label{eq:rate}
 \mathbb E\|\nabla F(y_\tau)\|_2^2
 \leq\frac{2(F(y_0)-F_{\inf})}{\gamma a\sqrt T}
      +\frac{L_FPa}{\gamma\sqrt T}+\frac{b^2}{\gamma^2}.
\end{equation}
\end{theorem}
With constants independent of $T$, the first two terms decay as $T^{-1/2}$; persistent direction error leaves the residual $b^2/\gamma^2$. This residual limits what the bound certifies; actual gradients may be smaller. A small gradient describes the sensitivity of the spatial average and does not certify optimal decisions. We derive the bias next, then extend the bound to momentum in Theorem~\ref{thm:momentum-rate}. Both results use a fixed search map, temperature, and probe law, conditionally unbiased batches, and independent rotations. The smoothness assumption is substantive: the radius randomization below suffices, while fixed radii alone need not. Annealing, refreshed subspaces, and fixed-order epochs require separate arguments.

\paragraph{Sampling assumptions.}
Assume $q\geq2$ and centered unit vertices. At each iteration, draw the batch, Haar rotations, and any randomized radii independently of one another and of the past, using fixed sampling laws. Rotations are independent across blocks; radii may share a common random multiplier. Require
$\mathbb E[\widehat{\mathcal L}_{\mathcal B}(y)\mid\mathcal H_t]=\mathcal L(y)$ for every $y$. Uniform sampling with replacement from the training data suffices. The same batch is used across candidates. Assume these random functions are jointly measurable and uniformly bounded on all required neighborhoods. Sampling without replacement within an epoch requires a separate argument for conditional unbiasedness.

\paragraph{The smoothing distribution.}
Let $B_q$ and $\mathbb S^{q-1}$ denote the unit ball and unit sphere in $\mathbb R^q$. Let $\nu$ be the law of a probe radius after choosing $k$ uniformly from $1,\ldots,K$, independently of rotations. Assume its support lies in $[\rho_{\min},\rho_{\max}]$ with $0<\rho_{\min}\leq\rho_{\max}<\infty$. For fixed radii, $\nu=K^{-1}\sum_k\delta_{\rho_k}$ and $\bar\rho=\mathbb E_\nu\rho$. Define a block smoothing operator
\begin{equation}\label{eq:kernel}
 (\mathcal K_j h)(y)=
 \int\frac{\rho}{\bar\rho}\,
 \mathbb E_{U\sim\operatorname{Unif}(B_q)}
 [h(y+\rho E_jU)]\,\nu(d\rho),
 \qquad G_j=\mathcal K_j\mathcal L,\quad
 F=\mathcal K_1\cdots\mathcal K_P\mathcal L.
\end{equation}
The mixture weights integrate to one. With fixed radii, they are $\rho_k/(K\bar\rho)$, giving the joint average in Eq.~\eqref{eq:joint-smoothing}. The use of ball averages follows from the spherical differentiation identity~\citep{flaxman2005online}, derived below.

For a bounded measurable loss and fixed radii, the following differentiation identity holds for almost every $y$:
\begin{equation}\label{eq:ball-identity}
 \mathbb E_{U\sim\operatorname{Unif}(\mathbb S^{q-1}),\,\rho\sim\nu}
 [\mathcal L(y+\rho E_jU)U]
 =\frac{\bar\rho}{q}\nabla_jG_j(y).
\end{equation}
To use this identity along an optimization trajectory, it must hold at the iterates. A sufficient condition for pointwise validity is smooth radius randomization: multiply every $\rho_k$ in a step by a common independent $1+\zeta$, where $\zeta$ has an even $C^\infty$ density compactly supported in $(-h,h)$, $0<h<1$. The mean radii are unchanged. If $p_\nu$ is the resulting radius density and $\omega_q$ is the unit-ball volume, the block kernel has density
\[
 k(u)=\frac1{\omega_q\bar\rho}
       \int_{\|u\|_2}^{\infty}p_\nu(r)r^{1-q}\,dr.
\]
It is smooth, compactly supported, and constant near the origin. Thus Eq.~\eqref{eq:ball-identity} holds pointwise, and the product kernel defining $F$ has bounded derivatives for a bounded loss. For other radius distributions, the theorem requires pointwise validity at the iterates and the stated smoothness of $F$. A shared probe-radius multiplier is sufficient because the expected linear term uses only marginal expectations; $F$ is still defined using independent block kernels.

\begin{proof}[Proof of Eq.~\eqref{eq:ball-identity}]
First take a smooth $h$. The divergence theorem applied to the $q$-ball gives
\[
 \nabla_j\mathbb E_{U\sim\operatorname{Unif}(B_q)}h(y+\rho E_jU)
 =\frac q\rho\mathbb E_{U\sim\operatorname{Unif}(\mathbb S^{q-1})}
 h(y+\rho E_jU)U.
\]
Multiplication by $\rho/\bar\rho$ and integration over $\nu$ give the identity. For bounded measurable $h$, approximation by smooth functions extends it as a weak derivative, hence almost everywhere.

The factor $q/\rho$ is the surface-to-volume ratio: a ball of radius $\rho$ has volume $\omega_q\rho^q$ and surface area $q\omega_q\rho^{q-1}$. This also explains the radius weights in Eq.~\eqref{eq:kernel}. Weighting the ball averages by $\rho/\bar\rho$ cancels the $1/\rho$ factor and recovers the uniform average over the $K$ probe radii.

For smooth radius randomization, a direct calculation proves pointwise validity even for a discontinuous loss. The kernel density above satisfies
\[
 \nabla k(u)=-\frac{p_\nu(\|u\|_2)}{\omega_q\bar\rho\,\|u\|_2^q}\,u
 \quad (u\ne0).
\]
Differentiating the convolution gives $\nabla_jG_j(y)=-\int\mathcal L(y+E_ju)\nabla k(u)\,du$. Polar integration, with sphere area $q\omega_q$, yields the right coefficient $q/\bar\rho$ in Eq.~\eqref{eq:ball-identity}. Boundedness of the loss and the smooth, compactly supported kernel justify differentiation under the integral. The product kernel has the same properties, so its second derivatives are integrable and $F$ has a bounded Hessian on the required neighborhoods.
\end{proof}

The ball identity involves costs multiplied by directions, whereas the algorithm uses exponential weights. The next lemma bounds the difference. It explains why cost spread relative to temperature controls the softmax contribution to the bias.

\begin{lemma}[Softmax linearization with a global remainder]\label{lem:linearization}
For $\varepsilon>0$ and centered unit directions $u_v$, let $D(C)=\sum_v\softmax(-C/\varepsilon)_vu_v$ and $\Delta=\max_vC_v-\min_vC_v$. Then
\begin{equation}\label{eq:linearization}
 \left\|D(C)+\frac1{V\varepsilon}\sum_vC_vu_v\right\|_2
 \leq\frac{\Delta^2}{4\varepsilon^2}.
\end{equation}
The bound holds at every spread; it can be large for low temperatures.
\end{lemma}
\begin{proof}
Put $z_v=-C_v/\varepsilon$, $p_v(t)=\softmax(tz)_v$, and $d(t)=\sum_vp_v(t)u_v$, for $0\leq t\leq1$. This interpolates between uniform weights and the actual softmax weights. Write $\mu(t)=\sum_vp_v(t)z_v$. Differentiating gives
\[
 p_v'(t)=p_v(t)(z_v-\mu(t)),\qquad
 \mu'(t)=\sum_vp_v(t)(z_v-\mu(t))^2
        =\operatorname{Var}_{p(t)}z.
\]
Hence $d(0)=0$ by centering, $d'(0)=V^{-1}\sum_vz_vu_v$, and
\[
 p_v''(t)=p_v(t)\big[(z_v-\mathbb E_{p(t)}z)^2-\operatorname{Var}_{p(t)}z\big].
\]
The unit-vector bound and the triangle inequality give
\[
 \|d''(t)\|_2\leq\sum_vp_v(t)
 \big[(z_v-\mu(t))^2+\operatorname{Var}_{p(t)}z\big]
 =2\operatorname{Var}_{p(t)}z.
\]
For $z_{\rm mid}=(\max z+\min z)/2$, the variance is at most
$\mathbb E_{p(t)}(z-z_{\rm mid})^2\leq(\max z-\min z)^2/4$.
Taylor's formula with integral remainder now yields
\[
 \|d(1)-d'(0)\|_2\leq\int_0^1(1-t)\|d''(t)\|_2\,dt
 \leq\frac{\Delta^2}{4\varepsilon^2}.
\]
\end{proof}

Assume Eq.~\eqref{eq:ball-identity} is pointwise valid at the iterates, and let $C^0_{j:}$ denote ideal batch probe costs before any additional evaluation error. Suppose, uniformly in $t$, almost surely,
\begin{align}\label{eq:error-assumptions}
 \mathbb E[(\max_vC^0_{jv}-\min_vC^0_{jv})^2\mid\mathcal H_t]&\leq \sigma_j^2,\nonumber\\
 \|\nabla_jG_j(y_t)-\nabla_jF(y_t)\|_2&\leq\kappa_j,\\
 \inf_{b_j\in\mathbb R}\|\widehat C_{j:}-C^0_{j:}-b_j\mathbf1\|_\infty&\leq\beta_j.\nonumber
\end{align}
These are bounds on the ideal cost spread, the block discrepancy, and the evaluation error after a common shift. They imply the bias model with the constant in Eq.~\eqref{eq:explicit-bias}.

\begin{proof}[Derivation of Eqs.~\eqref{eq:bias-model} and~\eqref{eq:explicit-bias}]
Each rotated vertex is uniform on the sphere. Conditional unbiasedness of the batch loss and Eq.~\eqref{eq:ball-identity} therefore give
\[
 \mathbb E\!\left[-\frac1{V\varepsilon}\sum_v C^0_{jv}R_jv_v\,\middle|\,\mathcal H_t\right]
 =-\frac1\varepsilon\mathbb E_{\rho,U}[\mathcal L(y_t+\rho E_jU)U]
 =-\gamma\nabla_jG_j(y_t).
\]
Lemma~\ref{lem:linearization} bounds the norm of the expected remainder by $\sigma_j^2/(4\varepsilon^2)$. Replacing $\nabla_jG_j$ by $\nabla_jF$ adds at most $\gamma\kappa_j$. Proposition~\ref{prop:stability} bounds the effect of using $\widehat C$ by $2\tanh(\beta_j/(2\varepsilon))$. Adding these errors within each block and summing their squares across the orthogonal blocks gives Eqs.~\eqref{eq:bias-model} and~\eqref{eq:explicit-bias}.
\end{proof}

For $g_t=\nabla F(y_t)$, the bias bound has the direct interpretation
\[
 \left\langle g_t,\mathbb E[D_t\mid\mathcal H_t]\right\rangle
 \leq-\gamma\|g_t\|_2^2+b\|g_t\|_2
 =-\gamma\|g_t\|_2\left(\|g_t\|_2-\frac b\gamma\right).
\]
This proves the descent threshold stated above. The next proof controls the additional error from taking a finite step.

\begin{proof}[Proof of Theorem~\ref{thm:rate}]
Each block direction is a convex combination of unit vectors, so $\|D_t\|_2^2\leq P$. Smoothness along the update segment and Eq.~\eqref{eq:bias-model} give, with $g_t=\nabla F(y_t)$,
\begin{align*}
 \mathbb E[F(y_{t+1})\mid\mathcal H_t]
 &\leq F(y_t)-s\gamma\|g_t\|_2^2+s\langle g_t,e_t\rangle
                 +\tfrac12L_Fs^2P\\
 &\leq F(y_t)-\tfrac12s\gamma\|g_t\|_2^2
                 +\frac{s b^2}{2\gamma}+\tfrac12L_Fs^2P.
\end{align*}
For the second line, Cauchy--Schwarz gives $\langle g_t,e_t\rangle\leq b\|g_t\|_2$, and the required scalar bound follows from a square:
\[
 \frac\gamma2\|g_t\|_2^2+\frac{b^2}{2\gamma}-b\|g_t\|_2
 =\frac{(\gamma\|g_t\|_2-b)^2}{2\gamma}\geq0.
\]
Summing expectations over $t<T$ makes the objective differences telescope:
\[
 \frac{s\gamma}{2}\sum_{t=0}^{T-1}\mathbb E\|g_t\|_2^2
 \leq F(y_0)-F_{\inf}+\frac{Ts b^2}{2\gamma}
                        +\frac{T L_Fs^2P}{2}.
\]
Divide by $s\gamma T/2$ and substitute $s=a/\sqrt T$. The independent uniform choice of $\tau$ turns the average on the left into $\mathbb E\|\nabla F(y_\tau)\|_2^2$, proving Eq.~\eqref{eq:rate}. For a random initialization with finite expected loss, replace $F(y_0)$ by its expectation.
\end{proof}

The bound separates the finite-iteration term from the persistent direction error. Reducing the smoothing radius may increase $L_F$ near a discontinuity, while lowering the temperature increases the error bounds in Eq.~\eqref{eq:explicit-bias}. A convergence analysis for decreasing radii or temperatures must control these changes jointly. The derivation uses the cost rows in Eq.~\eqref{eq:probes}: a fixed positive cost scale can be absorbed into $\varepsilon$, whereas data-dependent normalization requires its own conditional bias bound.

\paragraph{Interpreting the bias.}
For one block, $G_1=F$ and $\kappa_1=0$. With several blocks, $G_j$ averages only the active block, whereas $F$ averages all blocks. Their gradients can therefore differ. Minibatch randomness also contributes to $\sigma_j^2$, even when the batch loss is unbiased. Alternatively, under Eq.~\eqref{eq:batch-concentration}, linearize the exact mean row $C^\star$ instead: use its conditional spread bound for $\sigma_j^2$ and add $\chi_B$ from Eq.~\eqref{eq:batch-expected-error} to each block's bias bound. This separates the softmax approximation from sampling error. These are alternative bounds; the sampling error need not be counted twice.

If $\mathcal L(y)=\sum_j h_j(y_j)$ is additively separable across blocks, all $\kappa_j$ can also be zero: smoothing the other blocks changes only terms that are constant with respect to $y_j$. For a coupled loss, separate block averages need not have the partial derivatives of a common objective. The discrepancy term accounts for this difference.

With one block and exact evaluations, the normalized bias bound is $b/\gamma=q\sigma_1^2/(4\varepsilon\bar\rho)$. Increasing the temperature reduces this bound but also reduces the gradient coefficient $\gamma$. With several blocks, the discrepancy term need not vanish with temperature or iteration count.

For simulation estimates, Eq.~\eqref{eq:hoeffding} gives a high-probability error bound rather than the almost-sure bound in Eq.~\eqref{eq:error-assumptions}. If its failure probability is at most $\alpha$ conditional on the history and queries, the expected direction error is at most $\min\{2,2\tanh(\beta_j/(2\varepsilon))+2\alpha\}$. Use this quantity in place of the cost-error term in Eq.~\eqref{eq:explicit-bias}. On the success event, Proposition~\ref{prop:stability} applies; on failure, the direction difference is at most two.

\paragraph{Block discrepancy and decision changes.}
Under smooth radius randomization, let $\xi_{-j}$ have the product smoothing law in all blocks except $j$, with a zero $j$th block. Commuting the other block averages with the active-block derivative gives
\begin{align*}
 \|\nabla_jG_j(y)-\nabla_jF(y)\|_2
 &\leq\frac q{\bar\rho}\,
 \mathbb E\big|\mathcal L(y+\rho E_jU+\xi_{-j})
                 -\mathcal L(y+\rho E_jU)\big|,
\end{align*}
where $U$ is uniform on the sphere and all variables in this expectation are independent. Apply Eq.~\eqref{eq:ball-identity} before and after the other block averages, then use $\|U\|_2=1$. For the finite empirical decision-only loss with range width $W$, the expectation is at most $W$ times the probability that $\xi_{-j}$ changes at least one instance's returned decision. Rare changes therefore suffice for a small discrepancy. Frequent changes can also cancel, so this probability bound need not be tight.

\paragraph{Why fixed radii require a regularity assumption.}
Consider a two-decision OR problem with $\mathcal F=\{e_1,e_2\}$, predicted costs $(\|y\|_2^2-1,0)$, true costs $(1,0)$, and ties assigned to $e_1$. Its loss is $\mathcal L(y)=\mathbf1\{\|y\|_2\leq1\}$. For $P=1$, $q=2$, and radius $1$, the ball average is the normalized overlap of two unit disks. This overlap consists of two equal circular segments. Writing $r=\|y\|_2\leq2$ and integrating their cross-sections gives
\[
 F(y)=\frac4\pi\int_{r/2}^1\sqrt{1-t^2}\,dt
     =\frac2\pi\arccos(r/2)-\frac{r}{2\pi}\sqrt{4-r^2}
     =1-\frac2\pi r+O(r^3).
\]
Hence $F$ is not differentiable at $y=0$, although the predictor is smooth and the decision set is finite.

\subsection{Stationarity with scheduled momentum}\label{app:momentum-rate}
The bias model controls the current probe direction, while momentum also contains earlier directions. We account for this memory by adding to each iterate the displacement its stored velocity would produce over the remaining steps if no new directions were added. The resulting auxiliary sequence moves only in the current probe direction, so the same descent argument applies. This change of variables follows the approach used for stochastic momentum~\citep{yan2018unified}; the construction below allows a deterministic coefficient schedule.

\begin{theorem}[Finite-horizon bound with momentum]\label{thm:momentum-rate}
Run Eq.~\eqref{eq:momentum} for $T$ steps from deterministic $y_0$ and $v_0=0$, with $s_t=s=a/\sqrt T$, $a>0$, $\eta>0$, and a deterministic schedule $0\leq\mu_t\leq\bar\mu<1$. Assume Eq.~\eqref{eq:bias-model} holds conditional on the history including $v_t$, and $\|D_t\|_2\leq\sqrt P$. Define the deterministic coefficients and auxiliary points
\begin{equation}\label{eq:momentum-aux}
 c_T=0,\qquad c_t=\mu_t(1+c_{t+1}),\qquad
 h_t=\eta s(1+c_{t+1}),\qquad z_t=y_t+\eta c_tv_t.
\end{equation}
Assume $F\geq F_{\inf}$ and $F$ has $L_F$-Lipschitz gradient on a convex set containing all $y_t,z_t$ almost surely. Choose $\tau$ independently of the run with $\Pr(\tau=t)=h_t/\sum_{k<T}h_k$. Then
\begin{equation}\label{eq:momentum-rate}
 \mathbb E\|\nabla F(y_\tau)\|_2^2
 \leq\frac{2(F(y_0)-F_{\inf})}{\gamma\eta a\sqrt T}
 +\frac{L_FP\eta a(1+\bar\mu)}{\gamma(1-\bar\mu)^2\sqrt T}
 +\frac{b^2}{\gamma^2}.
\end{equation}
\end{theorem}
A linear momentum schedule is covered when its endpoints lie in $[0,\bar\mu]$ and the remaining assumptions hold. Momentum retains the $T^{-1/2}$ rate and the same bias residual for this smoothed objective. Its longer memory also enlarges the finite-step term, so this bound establishes no acceleration. Taking $\bar\mu=0$ and $\eta=1$ recovers Theorem~\ref{thm:rate}, including its uniform output index. The auxiliary points and the random index are used only in the analysis.

\begin{proof}
The backward recurrence and a geometric sum give
\[
 0\leq c_t\leq\frac{\bar\mu}{1-\bar\mu},\qquad
 \eta s\leq h_t\leq\frac{\eta s}{1-\bar\mu},\qquad
 \|v_t\|_2\leq\frac{s\sqrt P}{1-\bar\mu}.
\]
Consequently $\|z_t-y_t\|_2\leq R_s:=\eta s\bar\mu\sqrt P/(1-\bar\mu)^2$. Substituting Eq.~\eqref{eq:momentum} into the definition of $z_{t+1}$ cancels the stored velocity:
\[
 z_{t+1}=y_t+\eta(1+c_{t+1})(\mu_tv_t+sD_t)
        =z_t+h_tD_t.
\]
Also $z_0=y_0$ and $z_T=y_T$. Put $g_t=\nabla F(y_t)$. Smoothness and the bound on $D_t$ imply
\begin{align*}
 \mathbb E[F(z_{t+1})\mid\mathcal H_t]
 &\leq F(z_t)+h_t\langle g_t,\mathbb E[D_t\mid\mathcal H_t]\rangle
       +h_tL_FR_s\sqrt P+\tfrac12L_Fh_t^2P\\
 &\leq F(z_t)-\tfrac12h_t\gamma\|g_t\|_2^2
       +\frac{h_tb^2}{2\gamma}+h_tL_FR_s\sqrt P
       +\tfrac12L_Fh_t^2P.
\end{align*}
The first line bounds the gradient difference between $z_t$ and $y_t$ by $L_FR_s$. The second uses the bias model and the same scalar square inequality as the proof of Theorem~\ref{thm:rate}. Let $H_T=\sum_{t<T}h_t$. Summing expectations and dividing by $\gamma H_T/2$ yields
\[
 \sum_{t<T}\frac{h_t}{H_T}\mathbb E\|g_t\|_2^2
 \leq\frac{2(F(y_0)-F_{\inf})}{\gamma H_T}
      +\frac{b^2}{\gamma^2}+\frac{2L_FR_s\sqrt P}{\gamma}
      +\frac{L_FP\sum_{t<T}h_t^2}{\gamma H_T}.
\]
Use $H_T\geq\eta sT$ and $\sum_th_t^2/H_T\leq\eta s/(1-\bar\mu)$. The last two terms involving smoothness are at most
\[
 \frac{L_FP\eta s}{\gamma}
 \left(\frac{2\bar\mu}{(1-\bar\mu)^2}+\frac1{1-\bar\mu}\right)
 =\frac{L_FP\eta s(1+\bar\mu)}{\gamma(1-\bar\mu)^2}.
\]
Substitute $s=a/\sqrt T$ and the specified distribution of $\tau$ to obtain Eq.~\eqref{eq:momentum-rate}.
\end{proof}

The constants must be independent of $T$ for the stated rate. The lower bound and smoothness are needed on the auxiliary points as well as the actual iterates; smooth radius randomization and a globally bounded loss suffice. The result controls a weighted average of squared gradient norms, and hence their minimum in expectation, but not the last iterate or the cost of its decisions. Historical directions do not remove the softmax, block, or evaluation bias. Data-dependent momentum schedules and changing search bases are not covered by this proof.

\subsection{Arithmetic and storage costs}\label{app:complexity}
We use physical lengths in search coordinates: PolyStep's multiplier convention $r_p,r_s$ corresponds to $r=r_p\varepsilon$ and $s=r_s\varepsilon$~\citep{le2026polystep}.
Let $C_{\rm pred}(B)$ be the cost of a batch prediction, $C_S$ the average per-instance solve-and-score cost, and $C_A$ the cost of reconstructing one candidate. For simplex or orthoplex vertices, serial work per step is
\begin{equation}\label{eq:complexity}
 O\!\left(Q[C_A+C_{\rm pred}(B)+BC_S]+Pq^3+PVq\right).
\end{equation}
The last terms account for dense QR rotations and weighted directions; row softmax costs $O(PV)$. Momentum adds $O(d)$ work and storage, with no additional solver calls. A dense $D_\theta\times d$ subspace map can cost $O(D_\theta d)$ per reconstruction, while structured maps permit cheaper implementations. The cost matrix, accumulated update, and stored rotations need $O(PV+d+Pq^2)$ storage; candidate models, predictor activations, solver state, and the subspace map are additional costs. Streaming candidates avoids storing $Q$ complete models. Batching and parallel solver calls can reduce wall-clock time. Total computational comparisons should include both training evaluations and any reference-solution preprocessing required by a baseline.

\clearpage
\section{Additional experimental material}\label{app:additional-results}
This appendix provides detailed information on the optimization problems, extended experimental designs and results. Finally, we provide a list of hyperparameters configurations for reproducibility of the reported experiments.

\subsection{Extended OR metrics and results}\label{app:classical}
The technical details of the problems are taken from~\citep{mandi2024decision}.
\subsubsection{Problem Setups}

We evaluate 5 problems in the benchmark of \citet{mandi2024decision}, following the formulations in their Section~5.1. Table~\ref{tab:app-classical-problems} summarizes the settings, solvers, predictors, and data splits.

\paragraph{Shortest path on a $5\times5$ grid.}
Following \citet{mandi2024decision} and \citet{elmachtoub2022smart}, the problem finds a minimum-cost path from the southwest corner to the northeast corner of a $5\times5$ grid. The grid contains 25 nodes and 40 directed edges, each pointing north or east. The optimization problem is
\begin{equation*}
\min_{\bm{x}}\ \bm{c}^{\top}\bm{x}
\quad\text{s.t.}\quad
A\bm{x}=\bm{b},
\qquad
\bm{x}\geq\bm{0},
\end{equation*}
where $A\in\mathbb{R}^{|V|\times|E|}$ is the node-edge incidence matrix, with $+1$ at an edge's tail and $-1$ at its head. The vector $\bm{b}\in\mathbb{R}^{|V|}$ has value $1$ at the source, $-1$ at the sink, and $0$ elsewhere. The LP admits an integral optimal solution, in which $x_e=1$ indicates that edge $e$ is traversed. The cost vector $\bm{c}_i\in\mathbb{R}^{40}$ is generated from features $\bm{z}_i\sim\mathcal{N}(\bm{0},I_p)$ with $p=5$. For a fixed randomly generated matrix $B\in\mathbb{R}^{40\times p}$, each cost is computed as:
\begin{equation*}
c_{ij}
=
\left[
\left(
\frac{1}{\sqrt{p}}(B\bm{z}_i)_j+3
\right)^{\mathrm{Deg}}
+1
\right]\xi_{ij},
\end{equation*}
where $\xi_{ij}$ is drawn uniformly from $[1-\vartheta,1+\vartheta]$, with $\vartheta=0.5$. A linear predictor maps the features to the edge costs. The degree $\mathrm{Deg}\in\{1,2,4,6,8\}$ controls the nonlinearity of the underlying mapping and hence the misspecification of this predictor. Each setting uses 1,000 training, 250 validation, and 10,000 test instances.

\paragraph{Portfolio optimization.}
Following \citet{elmachtoub2022smart}, the problem allocates investments across $m$ assets to maximize return subject to a risk constraint,
\begin{equation*}
\max_{\bm{x}}\ \bm{c}^{\top}\bm{x}
\quad\text{s.t.}\quad
\bm{x}^{\top}\Sigma\bm{x}\leq\gamma,
\qquad
\mathbf{1}^{\top}\bm{x}\leq1,
\qquad
\bm{x}\geq\bm{0}.
\end{equation*}
Here, $x_j$ is the fraction invested in asset $j$, $\Sigma$ is the covariance matrix used to measure risk, and $\gamma$ is the risk limit. The features satisfy $\bm{z}_i\sim\mathcal{N}(\bm{0},I_p)$. A fixed matrix $B\in\{0,1\}^{m\times p}$ has independent Bernoulli entries with probability $0.5$. The conditional mean return is defined as:
\begin{equation*}
\bar{c}_{ij}
=
\left(
\frac{0.05}{\sqrt{p}}(B\bm{z}_i)_j
+0.1^{1/\mathrm{Deg}}
\right)^{\mathrm{Deg}}.
\end{equation*}
For noise magnitude $\vartheta$, a fixed matrix $L\in\mathbb{R}^{m\times4}$ has entries drawn uniformly from $[-0.0025\vartheta,0.0025\vartheta]$. The observed returns are then computed:
\begin{equation*}
\bm{c}_i
=
\bar{\bm{c}}_i
+L\bm{f}_i
+0.01\vartheta\bm{\xi}_i,
\end{equation*}
where $\bm{f}_i\sim\mathcal{N}(\bm{0},I_4)$ and $\bm{\xi}_i\sim\mathcal{N}(\bm{0},I_m)$ are independent of each other and of the features. Conditional on the fixed generator matrices, the return covariance given the features is $\operatorname{Cov}(\bm{c}_i\mid\bm{z}_i)=\Sigma=LL^{\top}+(0.01\vartheta)^2I_m$. The risk limit is $\gamma=2.25\,\bm{e}^{\top}\Sigma\bm{e}$, where $\bm{e}=\mathbf{1}/m$ is the equal-allocation portfolio. Both $\Sigma$ and $\gamma$ remain fixed within each setting. Our experiments use $m=50$ assets, $p=5$ features, $\mathrm{Deg}\in\{1,4,8,16\}$, and $\vartheta=1$. A linear predictor estimates the returns. Each setting uses 1,000 training, 250 validation, and 10,000 test instances.

\paragraph{Energy-cost aware scheduling.}
This problem schedules tasks on machines one day ahead using predicted electricity prices. Each task $j\in J$ has duration $\zeta_j$, earliest start $\zeta_j^{(1)}$, latest end $\zeta_j^{(2)}$, power consumption $\phi_j$, and resource requirements $u_{jw}$ for $w\in W$. Machine $i\in I$ has capacity $Q_{iw}$ for resource $w$. Tasks cannot be interrupted or transferred between machines. Let $T=\{0,\ldots,47\}$ index the half-hour slots, with task durations expressed in slots and latest completion times no greater than 48. The binary variable $x_{jit}$ indicates whether task $j$ starts on machine $i$ at time $t$. The schedule solves
\begin{equation*}
\begin{aligned}
\min_{\bm{x}}\quad
&\sum_{j\in J}\sum_{i\in I}\sum_{t\in T}
x_{jit}
\left(
\sum_{\substack{t'\in T\\t\leq t'<t+\zeta_j}}
\phi_j c_{t'}
\right)\\
\text{s.t.}\quad
&\sum_{i\in I}\sum_{t\in T}x_{jit}=1
&&\forall j\in J,\\
&x_{jit}=0
&&\forall j\in J,\ i\in I,\ t\in T
\text{ with }t<\zeta_j^{(1)},\\
&x_{jit}=0
&&\forall j\in J,\ i\in I,\ t\in T
\text{ with }t+\zeta_j>\zeta_j^{(2)},\\
&\sum_{j\in J}
\sum_{\substack{t'\in T\\t-\zeta_j<t'\leq t}}
x_{jit'}u_{jw}\leq Q_{iw}
&&\forall i\in I,\ w\in W,\ t\in T,\\
&x_{jit}\in\{0,1\}
&&\forall j\in J,\ i\in I,\ t\in T.
\end{aligned}
\end{equation*}
These constraints assign each task exactly once, respect its time window, and limit the resources used by simultaneously running tasks. The electricity prices come from the Irish Single Electricity Market Operator dataset~\citep{ifrim2012properties}, recorded at 30-minute intervals from 1 November 2011 to 31 December 2013. Each day forms one instance with 48 prices, each predicted from eight features by a shared linear model. Our three settings use three machines and 10, 15, or 20 tasks, denoted by loads $L=1$, $2$, and $3$. The 789 daily instances are reshuffled for each seed and split into 550 training, 100 validation, and 139 test instances.
\paragraph{Knapsack.}
The problem selects items to maximize their total value subject to a known capacity:
\begin{equation*}
\max_{\bm{x}\in\{0,1\}^{48}}\ \bm{c}^{\top}\bm{x}
\quad\text{s.t.}\quad
\bm{w}^{\top}\bm{x}\leq C,
\end{equation*}
where the item weights $\bm{w}$ and capacity $C$ are known, while the item values $\bm{c}$ are predicted. Each day of the electricity-price dataset forms one instance, with one item per half-hour slot. A fixed weight vector has entries in $\{3,5,7\}$ and total weight 240, and is shared across daily instances. Item values are constructed by multiplying each slot's electricity price by its weight and adding Gaussian noise $\xi\sim\mathcal{N}(0,25)$. A shared linear model predicts each value from the eight features associated with its slot. We evaluate capacities $C\in\{60,120,180\}$ and use the same 550/100/139 training, validation, and test split procedure as for scheduling.
\paragraph{Diverse bipartite matching.}
The problem matches articles between two disjoint sets $S_1$ and $S_2$, each containing 50 articles from the CORA citation network~\citep{sen2008collective}. Let $c_{ij}\in\{0,1\}$ indicate whether a citation exists between articles $i\in S_1$ and $j\in S_2$, and let $\phi_{ij}=1$ if they belong to the same field of study and $0$ otherwise. For objective coefficients $\bm{c}$, the matching solves
\begin{equation*}
\begin{aligned}
\max_{\bm{x}}\quad
&\sum_{i\in S_1}\sum_{j\in S_2}c_{ij}x_{ij}\\
\text{s.t.}\quad
&\sum_{j\in S_2}x_{ij}\leq1
&&\forall i\in S_1,\\
&\sum_{i\in S_1}x_{ij}\leq1
&&\forall j\in S_2,\\
&\sum_{i\in S_1}\sum_{j\in S_2}\phi_{ij}x_{ij}
\geq
\rho_1\sum_{i\in S_1}\sum_{j\in S_2}x_{ij},\\
&\sum_{i\in S_1}\sum_{j\in S_2}(1-\phi_{ij})x_{ij}
\geq
\rho_2\sum_{i\in S_1}\sum_{j\in S_2}x_{ij},\\
&x_{ij}\in\{0,1\}
&&\forall i\in S_1,\ j\in S_2.
\end{aligned}
\end{equation*}
At decision time, the predicted citation probabilities $\hat{c}_{ij}$ replace the unknown indicators $c_{ij}$ in the objective. The diversity constraints require minimum fractions $\rho_1$ and $\rho_2$ of selected pairs to belong to the same and different fields, respectively. The network is divided into 27 disjoint instances of 100 articles. Each article has 1,433 bag-of-words features. Concatenating the features of a pair gives a 2,866-dimensional input to a predictor with one hidden layer and a sigmoid output. We evaluate $(\rho_1,\rho_2)\in\{(0.1,0.1),(0.25,0.25),(0.5,0.5)\}$. 
\paragraph{Metric.} For minimization problems, regret is the realized objective minus the optimal objective, as in Eq.~\eqref{eq:regret}. For maximization problems, the order is reversed so that regret remains nonnegative. Following \citet{mandi2024decision}, we report mean relative regret for shortest path, scheduling, knapsack, and matching, dividing each instance's regret by its optimal objective value before averaging over the test set. Portfolio optimization uses absolute regret because the optimal return can be zero. We multiply portfolio regret by $10^3$ for presentation in the tables. For PolyStepOR, optimal reference solutions are used only to compute these evaluation metrics.
\begin{table}[ht]
\centering
\resizebox{0.95\textwidth}{!}{%
\begin{tabular}{lcccccc}
\toprule
Benchmark & Scenario & Cases & Problem & Predictor & Parameters & Train / Val / Test \\
\midrule
Shortest path & Degree & $\{1,2,4,6,8\}$ & LP & Linear & 240 & 1000 / 250 / 10000 \\
Portfolio & Degree & $\{1,4,8,16\}$ & QP & Linear & 300 & 1000 / 250 / 10000 \\
Knapsack & Capacity & $\{60,120,180\}$ & 0-1 ILP & Linear & 9 & 550 / 100 / 139 \\
Energy scheduling & Load & $\{1,2,3\}$ & ILP & Linear & 9 & 550 / 100 / 139 \\
Matching & Instance type & $\{1,2,3\}$ & ILP & MLP & 573{,}601 & 22 / 5 / 5 \\
\bottomrule
\end{tabular}}
\caption{Problems, solvers, predictors, and splits of the classical benchmarks of \citet{mandi2024decision}.}
\label{tab:app-classical-problems}
\end{table}

\subsubsection{Experimental Setups}
\paragraph{PolyStepOR.}
PolyStepOR minimizes the minibatch mean of the realized decision cost in Eq.~\eqref{eq:batch-loss}, reversing the objective sign for maximization problems. Training, early stopping, and checkpoint selection use no optimal reference decisions. We use one fixed configuration per problem without hyperparameter search (Table~\ref{tab:repro}). The predictor is initialized randomly using the evaluation seed and optimized through a fixed per-layer search basis, which spans the full parameter space for the linear predictors and is capped for the matching network. Each step evaluates $2Kd$ candidates on a shared minibatch sampled at that step, and matching uses all training instances.

Validation realized cost is evaluated approximately 20 times over the step budget, and training stops after 10 consecutive checks without improvement. We restore the checkpoint with the lowest validation cost and evaluate it once on the test set. Before computing the softmax weights, we subtract each row's minimum probe cost and apply mean-absolute-cost normalization. The common shift leaves the weights unchanged, whereas the normalization sets their effective temperature. The exact reference-cost invariance in Proposition~\ref{prop:reference} therefore requires the normalization scale to remain unchanged when reference costs are subtracted. The geometric results also have a restricted scope. Shortest path, knapsack, scheduling, and matching admit finite sets of discrete decisions, whereas portfolio has a continuous feasible set. Proposition~\ref{prop:margin} additionally requires the stated affine prediction map, so its exact formula applies to the linear predictors but not directly to the nonlinear matching network.

\paragraph{Baselines.} We run MSE, SPO+~\citep{elmachtoub2022smart}, DBB~\citep{pogancic2020differentiation}, I-MLE~\citep{niepert2021implicit}, FY~\citep{blondel2020learning}, and the solution-caching losses Listwise, Pairwise, Pairwise(diff), and MAP~\citep{mulamba2020contrastive,mandi2022decision} using the benchmark code and hyperparameters selected by \citet{mandi2024decision}. These methods use Adam, a plateau learning-rate schedule, and the original epoch budgets, returning the final-epoch model. The caching losses initialize their solution cache with all training optima and solve $5\%$ of instances per step. 

\paragraph{Computational cost.}
Table~\ref{tab:app-classical-compute} reports the cost of an example PolyStepOR run. Each step evaluates $2Kd$ candidates on $B$ shared training instances, requiring $2KdB$ solver calls before validation. These solves dominate training time and use the number of solver processes listed in the table. Portfolio solves run in a single-threaded solver within the training process.

\begin{table}[ht]
\centering
\setlength{\tabcolsep}{3pt}
\resizebox{0.95\textwidth}{!}{%
\begin{tabular}{lcccccccc}
\toprule
\multicolumn{1}{c}{Problem} & $d$ & $2Kd$ & $B$ & \begin{tabular}[c]{@{}c@{}}Solver calls\\ per step\end{tabular} & \begin{tabular}[c]{@{}c@{}}Mean stop /\\ budget (steps)\end{tabular} & \begin{tabular}[c]{@{}c@{}}Solver\\ processes\end{tabular} & \begin{tabular}[c]{@{}c@{}}Training\\ (h)\end{tabular} & \begin{tabular}[c]{@{}c@{}}Inference per\\ instance (ms)\end{tabular} \\
\midrule
Shortest path & 240 & 480 & 128 & \num{61440} & 139.5 / 150 & 8 & $0.35\pm0.05$ & 2.36 \\
Portfolio & 300 & 600 & 128 & \num{76800} & 92.5 / 100 & 1 & $2.04\pm1.03$ & 2.15 \\
Knapsack & 9 & 18 & 128 & \num{2304} & 98.0 / 100 & 1 & $0.06\pm0.01$ & 0.89 \\
Energy & 9 & 18 & 128 & \num{2304} & 93.8 / 100 & 32 & $1.46\pm0.46$ & 288.14 \\
Matching & 512 & 1024 & 17 & \num{17408} & 80.7 / 100 & 32 & $1.25\pm0.90$ & 143.25 \\
\bottomrule
\end{tabular}}
\caption{Cost of one PolyStepOR run on the classical benchmarks, with training time as mean $\pm$ standard deviation over the settings and ten seeds. Inference is one forward pass plus one solve, measured on load~2 for energy.}
\label{tab:app-classical-compute}
\end{table}

\subsubsection{Extended Results}
Table~\ref{tab:app-classical-pooled} reports regret averaged across the settings of each problem. Tables~\ref{tab:app-classical-settings} and~\ref{tab:app-classical-settings-real} report individual settings, and Table~\ref{tab:app-classical-tests} provides paired tests over the ten evaluation seeds. Reported uncertainties are standard deviations across seeds. The lowest-mean baseline in the test table is selected from the methods shown in the main-text figures, which exclude the solution-caching methods included in the extended result tables.

\begin{table}[ht]
\centering
\caption{Test regret on the classical benchmarks averaged on all settings of each problem over ten seeds. Bold underlined marks the lowest mean in each column and bold the second lowest.}
\label{tab:app-classical-pooled}
\vspace{4pt}
\resizebox{0.95\textwidth}{!}{%
\begin{tabular}{lccccc}
\toprule
Method & Shortest path & Portfolio ($\times10^{3}$) & Knapsack & Energy & Matching \\
\midrule
MSE & 0.1470 $\pm$ 0.0029 & 1.0738 $\pm$ 0.0008 & 0.0870 $\pm$ 0.0089 & 0.0224 $\pm$ 0.0021 & 0.9352 $\pm$ 0.0060 \\
SPO+ & 0.1303 $\pm$ 0.0002 & 0.2617 $\pm$ 0.0315 & 0.0829 $\pm$ 0.0183 & 0.0151 $\pm$ 0.0014 & 0.9177 $\pm$ 0.0130 \\
DBB & 0.1228 $\pm$ 0.0033 & 141.5336 $\pm$ 0.0659 & \second{0.0616 $\pm$ 0.0016} & \second{0.0150 $\pm$ 0.0015} & 0.9171 $\pm$ 0.0098 \\
I-MLE & \second{0.1073 $\pm$ 0.0007} & 131.6175 $\pm$ 1.2281 & 0.0620 $\pm$ 0.0024 & \best{0.0149 $\pm$ 0.0015} & 0.9165 $\pm$ 0.0174 \\
FY & \best{0.1034 $\pm$ 0.0007} & 107.8452 $\pm$ 1.0827 & 0.1467 $\pm$ 0.0152 & 0.0151 $\pm$ 0.0016 & 0.9304 $\pm$ 0.0139 \\
Listwise & 0.1128 $\pm$ 0.0006 & 9.2482 $\pm$ 7.2742 & 0.1493 $\pm$ 0.0274 & 0.0277 $\pm$ 0.0037 & \second{0.9005 $\pm$ 0.0148} \\
Pairwise & 0.1140 $\pm$ 0.0027 & \second{0.1109 $\pm$ 0.0155} & 0.2270 $\pm$ 0.0563 & 0.0276 $\pm$ 0.0031 & \best{0.8979 $\pm$ 0.0156} \\
Pairwise(diff) & 0.3562 $\pm$ 0.0049 & 0.5248 $\pm$ 1.1732 & 0.1074 $\pm$ 0.0294 & 0.0170 $\pm$ 0.0021 & 0.9033 $\pm$ 0.0104 \\
MAP & 0.1367 $\pm$ 0.0022 & \best{0.0955 $\pm$ 0.0117} & 0.1328 $\pm$ 0.0248 & 0.0152 $\pm$ 0.0013 & 0.9088 $\pm$ 0.0092 \\
\midrule
PolyStepOR & 0.1121 $\pm$ 0.0016 & 0.1602 $\pm$ 0.0309 & \best{0.0611 $\pm$ 0.0016} & \second{0.0150 $\pm$ 0.0016} & 0.9223 $\pm$ 0.0107 \\
\bottomrule
\end{tabular}}
\end{table}

\begin{table}[t]
\centering
\renewcommand{\arraystretch}{0.92}
\caption{Test regret per setting on the synthetic classical benchmarks over ten seeds (portfolio in absolute regret $\times10^{3}$). Bold underlined marks the lowest mean in each column and bold the second lowest.}
\label{tab:app-classical-settings}
\vspace{4pt}
\resizebox{0.95\textwidth}{!}{%
\footnotesize\begin{tabular*}{\textwidth}{@{\extracolsep{\fill}}lrrrrr@{}}
\toprule
\multirow{2}{*}{Method} & \multicolumn{5}{c}{Shortest path, degree $\mathrm{deg}$} \\
\cmidrule(lr){2-6}
 & $\mathrm{deg}=1$ & $\mathrm{deg}=2$ & $\mathrm{deg}=4$ & $\mathrm{deg}=6$ & $\mathrm{deg}=8$ \\
\midrule
MSE & \best{0.153 $\pm$ 0.000} & \second{0.101 $\pm$ 0.000} & 0.085 $\pm$ 0.001 & 0.131 $\pm$ 0.002 & 0.266 $\pm$ 0.013 \\
SPO+ & \second{0.154 $\pm$ 0.000} & \best{0.101 $\pm$ 0.000} & 0.076 $\pm$ 0.000 & 0.099 $\pm$ 0.001 & 0.222 $\pm$ 0.000 \\
DBB & 0.164 $\pm$ 0.003 & 0.126 $\pm$ 0.009 & 0.083 $\pm$ 0.002 & 0.092 $\pm$ 0.005 & 0.148 $\pm$ 0.006 \\
I-MLE & 0.159 $\pm$ 0.002 & 0.104 $\pm$ 0.002 & \second{0.071 $\pm$ 0.000} & \second{0.087 $\pm$ 0.001} & \second{0.116 $\pm$ 0.002} \\
FY & 0.157 $\pm$ 0.001 & 0.102 $\pm$ 0.000 & \best{0.071 $\pm$ 0.001} & \best{0.081 $\pm$ 0.001} & \best{0.106 $\pm$ 0.002} \\
Listwise & 0.169 $\pm$ 0.001 & 0.104 $\pm$ 0.000 & 0.077 $\pm$ 0.002 & 0.089 $\pm$ 0.000 & 0.126 $\pm$ 0.001 \\
Pairwise & 0.171 $\pm$ 0.014 & 0.105 $\pm$ 0.000 & 0.078 $\pm$ 0.001 & 0.091 $\pm$ 0.000 & 0.124 $\pm$ 0.001 \\
Pairwise(diff) & 0.163 $\pm$ 0.000 & 0.117 $\pm$ 0.000 & 0.215 $\pm$ 0.004 & 0.340 $\pm$ 0.008 & 0.946 $\pm$ 0.019 \\
MAP & 0.156 $\pm$ 0.000 & 0.103 $\pm$ 0.000 & 0.074 $\pm$ 0.001 & 0.119 $\pm$ 0.002 & 0.232 $\pm$ 0.010 \\
\midrule
PolyStepOR & 0.157 $\pm$ 0.001 & 0.105 $\pm$ 0.002 & 0.075 $\pm$ 0.001 & 0.088 $\pm$ 0.002 & 0.134 $\pm$ 0.008 \\
\bottomrule
\end{tabular*}}
\vspace{4pt}
\resizebox{0.95\textwidth}{!}{%
\footnotesize\begin{tabular*}{\textwidth}{@{\extracolsep{\fill}}lrrrr@{}}
\toprule
\multirow{2}{*}{Method} & \multicolumn{4}{c}{Portfolio, degree $\mathrm{deg}$, absolute regret ($\times10^{3}$)} \\
\cmidrule(lr){2-5}
 & $\mathrm{deg}=1$ & $\mathrm{deg}=4$ & $\mathrm{deg}=8$ & $\mathrm{deg}=16$ \\
\midrule
MSE & \best{0.116 $\pm$ 0.003} & 0.038 $\pm$ 0.000 & 0.565 $\pm$ 0.000 & 3.576 $\pm$ 0.000 \\
SPO+ & 0.261 $\pm$ 0.025 & 0.292 $\pm$ 0.091 & 0.196 $\pm$ 0.052 & 0.297 $\pm$ 0.045 \\
DBB & 126.527 $\pm$ 0.004 & 119.848 $\pm$ 0.003 & 134.911 $\pm$ 0.002 & 184.848 $\pm$ 0.265 \\
I-MLE & 114.044 $\pm$ 0.801 & 108.812 $\pm$ 0.443 & 122.268 $\pm$ 0.672 & 181.346 $\pm$ 4.286 \\
FY & 102.869 $\pm$ 1.013 & 101.373 $\pm$ 1.057 & 104.740 $\pm$ 1.524 & 122.399 $\pm$ 4.397 \\
Listwise & 18.994 $\pm$ 13.899 & 7.865 $\pm$ 12.390 & 9.753 $\pm$ 15.096 & 0.381 $\pm$ 0.172 \\
Pairwise & 0.371 $\pm$ 0.063 & \best{0.000 $\pm$ 0.000} & \second{0.011 $\pm$ 0.017} & \best{0.062 $\pm$ 0.003} \\
Pairwise(diff) & 0.406 $\pm$ 0.006 & \second{0.000 $\pm$ 0.000} & \best{0.001 $\pm$ 0.000} & 1.692 $\pm$ 4.687 \\
MAP & \second{0.195 $\pm$ 0.043} & 0.010 $\pm$ 0.025 & 0.022 $\pm$ 0.016 & 0.155 $\pm$ 0.022 \\
\midrule
PolyStepOR & 0.423 $\pm$ 0.049 & 0.055 $\pm$ 0.048 & 0.056 $\pm$ 0.077 & \second{0.107 $\pm$ 0.042} \\
\bottomrule
\end{tabular*}}
\end{table}

\begin{table}[t]
\centering
\renewcommand{\arraystretch}{0.92}
\caption{Test regret per setting on the classical benchmarks with real data over ten seeds. Bold underlined marks the lowest mean in each column and bold the second lowest.}
\label{tab:app-classical-settings-real}
\vspace{4pt}
\resizebox{0.95\textwidth}{!}{%
\footnotesize\begin{tabular*}{\textwidth}{@{\extracolsep{\fill}}lrrr@{}}
\toprule
\multirow{2}{*}{Method} & \multicolumn{3}{c}{Knapsack, capacity $C$} \\
\cmidrule(lr){2-4}
 & $C=60$ & $C=120$ & $C=180$ \\
\midrule
MSE & 0.149 $\pm$ 0.023 & 0.077 $\pm$ 0.005 & 0.035 $\pm$ 0.004 \\
SPO+ & 0.142 $\pm$ 0.030 & 0.068 $\pm$ 0.030 & 0.038 $\pm$ 0.035 \\
DBB & 0.110 $\pm$ 0.006 & \best{0.052 $\pm$ 0.003} & \best{0.023 $\pm$ 0.002} \\
I-MLE & \second{0.108 $\pm$ 0.005} & 0.053 $\pm$ 0.004 & 0.025 $\pm$ 0.002 \\
FY & 0.289 $\pm$ 0.033 & 0.109 $\pm$ 0.016 & 0.043 $\pm$ 0.004 \\
Listwise & 0.275 $\pm$ 0.060 & 0.131 $\pm$ 0.044 & 0.042 $\pm$ 0.021 \\
Pairwise & 0.302 $\pm$ 0.066 & 0.217 $\pm$ 0.051 & 0.161 $\pm$ 0.142 \\
Pairwise(diff) & 0.167 $\pm$ 0.028 & 0.100 $\pm$ 0.052 & 0.056 $\pm$ 0.045 \\
MAP & 0.164 $\pm$ 0.013 & 0.104 $\pm$ 0.043 & 0.130 $\pm$ 0.061 \\
\midrule
PolyStepOR & \best{0.108 $\pm$ 0.004} & \second{0.052 $\pm$ 0.004} & \second{0.024 $\pm$ 0.002} \\
\bottomrule
\end{tabular*}}
\vspace{4pt}
\resizebox{0.95\textwidth}{!}{%
\footnotesize\begin{tabular*}{\textwidth}{@{\extracolsep{\fill}}lrrr@{}}
\toprule
\multirow{2}{*}{Method} & \multicolumn{3}{c}{Energy-cost aware scheduling, load $L$} \\
\cmidrule(lr){2-4}
 & $L=1$ & $L=2$ & $L=3$ \\
\midrule
MSE & 0.021 $\pm$ 0.003 & 0.030 $\pm$ 0.002 & 0.017 $\pm$ 0.002 \\
SPO+ & 0.015 $\pm$ 0.002 & 0.020 $\pm$ 0.002 & 0.010 $\pm$ 0.001 \\
DBB & \best{0.015 $\pm$ 0.002} & 0.020 $\pm$ 0.003 & 0.010 $\pm$ 0.002 \\
I-MLE & \second{0.015 $\pm$ 0.002} & 0.020 $\pm$ 0.003 & \best{0.010 $\pm$ 0.002} \\
FY & 0.016 $\pm$ 0.002 & \best{0.019 $\pm$ 0.002} & 0.010 $\pm$ 0.002 \\
Listwise & 0.024 $\pm$ 0.004 & 0.028 $\pm$ 0.004 & 0.031 $\pm$ 0.007 \\
Pairwise & 0.032 $\pm$ 0.003 & 0.030 $\pm$ 0.005 & 0.020 $\pm$ 0.003 \\
Pairwise(diff) & 0.016 $\pm$ 0.002 & 0.022 $\pm$ 0.003 & 0.013 $\pm$ 0.003 \\
MAP & 0.015 $\pm$ 0.002 & 0.020 $\pm$ 0.002 & 0.010 $\pm$ 0.001 \\
\midrule
PolyStepOR & 0.015 $\pm$ 0.002 & \second{0.020 $\pm$ 0.002} & \second{0.010 $\pm$ 0.002} \\
\bottomrule
\end{tabular*}}
\vspace{4pt}
\resizebox{0.95\textwidth}{!}{%
\footnotesize\begin{tabular*}{\textwidth}{@{\extracolsep{\fill}}lrrr@{}}
\toprule
\multirow{2}{*}{Method} & \multicolumn{3}{c}{Diverse bipartite matching, instance type $I$} \\
\cmidrule(lr){2-4}
 & $I=1$ & $I=2$ & $I=3$ \\
\midrule
MSE & 0.934 $\pm$ 0.010 & 0.933 $\pm$ 0.007 & 0.939 $\pm$ 0.007 \\
SPO+ & 0.924 $\pm$ 0.017 & 0.924 $\pm$ 0.025 & 0.905 $\pm$ 0.016 \\
DBB & 0.917 $\pm$ 0.022 & 0.914 $\pm$ 0.021 & 0.921 $\pm$ 0.018 \\
I-MLE & 0.919 $\pm$ 0.025 & 0.911 $\pm$ 0.024 & 0.920 $\pm$ 0.016 \\
FY & 0.939 $\pm$ 0.020 & 0.929 $\pm$ 0.025 & 0.924 $\pm$ 0.016 \\
Listwise & \second{0.916 $\pm$ 0.030} & 0.907 $\pm$ 0.017 & \best{0.879 $\pm$ 0.018} \\
Pairwise & \best{0.908 $\pm$ 0.012} & \second{0.905 $\pm$ 0.024} & 0.881 $\pm$ 0.019 \\
Pairwise(diff) & 0.916 $\pm$ 0.018 & \best{0.903 $\pm$ 0.020} & 0.891 $\pm$ 0.029 \\
MAP & 0.928 $\pm$ 0.021 & 0.919 $\pm$ 0.020 & \second{0.880 $\pm$ 0.018} \\
\midrule
PolyStepOR & 0.924 $\pm$ 0.015 & 0.928 $\pm$ 0.014 & 0.915 $\pm$ 0.024 \\
\bottomrule
\end{tabular*}}
\end{table}

\begin{table}[ht]
\centering
\begin{tabular}{lccccc}
\toprule
Problem & Setting & Lowest-mean baseline & Lower mean & $p$ Wilcoxon & $p$ sign-flip \\
\midrule
\multirow{6}{*}{Shortest path} & $\mathrm{deg}=1$ & MSE & MSE & 0.002 & 0.002 \\
 & $\mathrm{deg}=2$ & SPO+ & SPO+ & 0.002 & 0.002 \\
 & $\mathrm{deg}=4$ & FY & FY & 0.002 & 0.002 \\
 & $\mathrm{deg}=6$ & FY & FY & 0.002 & 0.002 \\
 & $\mathrm{deg}=8$ & FY & FY & 0.002 & 0.002 \\
 & pooled & FY & FY & 0.002 & 0.002 \\
\midrule
\multirow{5}{*}{Portfolio} & $\mathrm{deg}=1$ & MSE & MSE & 0.002 & 0.002 \\
 & $\mathrm{deg}=4$ & MSE & MSE & 0.846 & 0.393 \\
 & $\mathrm{deg}=8$ & SPO+ & PolyStepOR & 0.006 & 0.006 \\
 & $\mathrm{deg}=16$ & SPO+ & PolyStepOR & 0.002 & 0.002 \\
 & pooled & SPO+ & PolyStepOR & 0.002 & 0.002 \\
\midrule
\multirow{4}{*}{Knapsack} & $C=60$ & I-MLE & PolyStepOR & 0.322 & 0.518 \\
 & $C=120$ & DBB & DBB & 0.557 & 0.465 \\
 & $C=180$ & DBB & DBB & 0.695 & 0.514 \\
 & pooled & DBB & PolyStepOR & 0.322 & 0.268 \\
\midrule
\multirow{4}{*}{Energy} & $L=1$ & DBB & DBB & 0.322 & 0.303 \\
 & $L=2$ & FY & FY & 0.193 & 0.068 \\
 & $L=3$ & I-MLE & I-MLE & 0.432 & 0.488 \\
 & pooled & I-MLE & I-MLE & 1.000 & 0.496 \\
\midrule
\multirow{4}{*}{Matching} & $I=1$ & DBB & DBB & 0.625 & 0.471 \\
 & $I=2$ & I-MLE & I-MLE & 0.131 & 0.059 \\
 & $I=3$ & SPO+ & SPO+ & 0.275 & 0.328 \\
 & pooled & I-MLE & I-MLE & 0.557 & 0.352 \\
\bottomrule
\end{tabular}
\caption{Paired two-sided tests of PolyStepOR against the lowest-mean baseline of the main-text figures, over the ten paired seeds of each setting. ``Lower mean" indicates the method with the lower mean of the pair.}
\label{tab:app-classical-tests}
\end{table}

\paragraph{Portfolio.}
Portfolio optimization has a continuous feasible set and a quadratic risk constraint. Averaged across polynomial degrees, PolyStepOR achieves the lowest regret among the methods in the main-text figures ($p=0.002$ against SPO+). In the complete comparison, however, MAP and Pairwise achieve lower average regret. DBB, I-MLE, and FY produce substantially higher regret, consistent with the qualitative findings of \citet{mandi2024decision}. MSE performs best at degree~1, where a linear model can represent the conditional mean return.
\paragraph{Shortest path.}
The selected path changes only when predictions cross a decision boundary, so probes that preserve the path provide no cost difference for training. FY uses the optimal path as its target, while DBB and I-MLE receive the true cost vector as the derivative with respect to the decision. PolyStepOR receives only the scalar minibatch cost of each candidate. These differences describe the available training signals but do not by themselves establish the cause of the performance gap. PolyStepOR's regret exceeds that of the lowest-mean baseline by approximately $0.004$ at degree~1 and $0.028$ at degree~8. Training reaches the 150-step budget without early stopping in 24 of the 50 runs, including nine of the ten degree-8 runs. This indicates that the stopping criterion was not reached within the available budget, although it does not establish that longer training would close the gap. PolyStepOR nevertheless has lower mean regret than MSE at degrees 4, 6, and 8.
\paragraph{Matching.}
Averaged across instance types, the methods in the main-text figures remain within $2.0\%$ of MSE, while Pairwise reduces regret by approximately $4.0\%$ relative to MSE. Mean relative regret exceeds $0.87$ in every reported setting, indicating that all methods remain far from the optimal matching objective. Interpretation is limited by the five test instances per type and the different training splits. 
\paragraph{Knapsack and scheduling.}
Across all six settings, we detect no significant difference between PolyStepOR and the baseline with the lowest mean regret in the main-text comparison, with paired $p\geq0.19$. PolyStepOR also has the lowest knapsack regret averaged across capacities among all ten methods. Both problems use linear predictors with nine parameters, allowing the full parameter space to be probed with 18 candidates per step. This makes each update substantially less expensive in solver calls than for the higher-dimensional predictors.

\subsection{Extended in-prediction OR metrics and results}\label{app:inconstraint}
\subsubsection{Problem Setups}
We consider the integer optimization problems with predicted constraints introduced by \citet{silvestri2026score}. Predictions determine the first-stage decision, which is evaluated under the realized parameters and corrected through recourse when necessary. The published data-generation procedure uses the mapping of \citet{elmachtoub2022smart}, with degree $5$, $p=5$ features, and noise half-width $0.5$, to generate Poisson distribution parameters. Realized problem parameters are then sampled from these distributions. Differences between this description and the released data are documented below. Each problem uses five datasets with three random splits each, giving 15 runs per setting. We retain the benchmark's notation $\rho$ for the recourse penalty in this subsection.

\paragraph{Knapsack with stochastic weights (KP weight).} Given a set of items $I$ with known values $v_i$ and capacity $C$, the first stage selects items using the predicted weights $\hat{\bm{w}}$,
\begin{equation*}
\hat{\bm{x}}\in\arg\max_{\bm{x}\in\{0,1\}^{|I|}}\Bigl\{\sum_{i\in I}v_ix_i:\ \sum_{i\in I}\hat w_ix_i\leq C\Bigr\}.
\end{equation*}
After observing the true weights $\bm{w}$, recourse restores feasibility by choosing additions $\bm{u}^+$ and removals $\bm{u}^-$,
\begin{equation}\label{eq:app-kp-recourse}
\max_{\bm{u}^+,\bm{u}^-\in\{0,1\}^{|I|}}\ \sum_{i\in I}\Bigl(\tfrac{1}{\rho}v_iu^+_i-\rho v_iu^-_i\Bigr)\quad\text{s.t.}\quad \sum_{i\in I}w_i(\hat x_i+u^+_i-u^-_i)\leq C,\quad \hat{\bm{x}}\geq\bm{u}^-,\quad \hat{\bm{x}}+\bm{u}^+\leq\mathbf{1}.
\end{equation}
The binary vectors $\bm{u}^+$ and $\bm{u}^-$ identify the items added and removed. Relative to the initial selection value, each addition contributes $v_i/\rho$ and each removal incurs $\rho v_i$, where $\rho>1$. The total post-hoc value is the initial selection value plus the optimized recourse objective. The problem has 50 items and $\rho\in\{5,10,20\}$.

\paragraph{Knapsack with stochastic capacity (KP capacity).} This variant predicts the capacity while keeping item weights and values known. The first-stage constraint uses $\hat C$, and recourse uses the realized capacity $C$ in Eq.~\eqref{eq:app-kp-recourse} (Table~8 of \citet{silvestri2026score}). The problem has 50 items and $\rho\in\{5,10,20\}$.

\paragraph{Weighted set multi-cover (WSMC).} Given items $I$ and covers $J$, let $a_{ij}=1$ if cover $j$ covers item $i$, let $\kappa_j$ be the cost of cover $j$, and let $d_i$ be the coverage requirement of item $i$. The first stage solves
\begin{equation*}
\hat{\bm{x}}\in\arg\min_{\bm{x}\in\mathbb{Z}_{\geq0}^{|J|}}\Bigl\{\sum_{j\in J}\kappa_jx_j:\ \sum_{j\in J}a_{ij}x_j\geq\hat d_i\ \ \forall i\in I\Bigr\},
\end{equation*}
where $x_j$ counts how often cover $j$ is selected and the requirements $\hat{\bm{d}}$ are predicted. After the true requirements $\bm{d}$ are observed, each unit of unmet demand for item $i$ incurs a recourse charge equal to $\rho$ times the largest cost among covers containing that item. The resulting post-hoc cost is
\begin{equation}\label{eq:app-wsmc-recourse}
\sum_{j\in J}\kappa_j\hat x_j+\rho\sum_{i\in I}\Bigl(\max_{j:\,a_{ij}=1}\kappa_j\Bigr)\max\Bigl(0,\ d_i-\sum_{j\in J}a_{ij}\hat x_j\Bigr).
\end{equation}
The availability matrices follow the generator of \citet{silvestri2026score}, and the cover costs are drawn uniformly from $[1,100]$. We use 10 items and 50 covers with $\rho\in\{1,5,10\}$, as in Table~3 of \citet{silvestri2026score}, and 5 items and 25 covers, as in their Table~9.

\paragraph{Metrics.} The primary metric of \citet{silvestri2026score} is the relative post-hoc regret. It divides the post-hoc regret of Eq.~\eqref{eq:pregret} by the magnitude of the optimal cost under the true parameters and averages over the test instances. We also report the fraction of test instances whose first-stage decision is infeasible under the realized parameters, before recourse, and the mean squared prediction error (MSE). For knapsack, regret uses the maximization convention, subtracting the achieved post-hoc value from the optimal value.

\subsubsection{Experimental Setups}
\paragraph{Data and splits.} We use the authors' code and data generators. The released code splits each dataset into 720 training, 80 validation, and 50 test instances, while the paper states an 80/10/10 split. The released knapsack data use multiplicative noise $0.1$ and additive noise $0.03$ instead of the stated half-width $0.5$. On these data, our baseline reruns, including CombOptNet where available, reproduce the knapsack results of Tables~2 and~8 of \citet{silvestri2026score} within one standard deviation. We generate the WSMC data with the authors' generator at the stated half-width $0.5$. Our reruns then reproduce their Table~9 for $5\times25$, whereas the MSE and SFGE values of their Table~3 for $10\times50$ are about twice ours.

\paragraph{PolyStepOR.}
We train the benchmark's linear predictors using the original data loaders and optimization solvers. The training loss is the post-hoc cost in Eq.~\eqref{eq:recourse-loss}, with the optimal reference term omitted. Validation uses the same cost for early stopping and checkpoint selection, so neither training nor selection requires optimal reference decisions.

All settings use one fixed configuration without hyperparameter search. We use a fixed per-layer search basis without a dimension cap. The predictor's output scaling remains fixed. When $d$ is not a multiple of $q$, the last of the $\lceil d/q\rceil$ blocks is filled with dummy coordinates that map to no parameter, so probes along them return the current cost. Each step evaluates candidates on a shared minibatch, following the fixed order of the original data loader. Training stops after 10 consecutive epochs without improvement in validation post-hoc cost. We restore the best validation checkpoint and evaluate it once on the test set. Candidate decisions are computed using the authors' Gurobi models, with a 30-second time limit for knapsack. The fixed batch order differs from the independent sampling assumed in the stationarity analysis.
\paragraph{Baselines.}
The baseline denoted MSE follows \citet{silvestri2026score}. It fits a Gaussian predictive distribution by minimizing negative log-likelihood, with a standard deviation that does not depend on the features. Thus, its training objective is likelihood-based despite its benchmark label. CombOptNet~\citep{paulus2021comboptnet} uses the authors' implementation. Its released WSMC loader provides no validation split, so we add one following the knapsack loader's procedure. We do not evaluate CombOptNet on KP capacity, for which the original paper reports no result.
SFGE uses the authors' settings, including Gaussian predictions with a learned standard deviation, standardized batch values, and Adam with learning rate $0.005$. Early stopping is based on validation regret. In these implementations, CombOptNet trains on optimal reference decisions and SFGE evaluates an objective containing optimal reference costs. MSE and PolyStepOR use neither. This describes the reference usage of the evaluated implementations, rather than an inherent requirement of score-function estimation.
\paragraph{Computational cost.}
KP capacity uses five CPU cores with four parallel solver processes. The remaining problems use eight cores and eight solver processes. Memory use is approximately 4~GB per run. Across penalty settings, the mean stopping epoch ranges from 18.6 to 19.7 for KP capacity, 28.2 to 29.3 for KP weight, 17.7 to 23.7 for WSMC $10\times50$, and 19.6 to 25.4 for WSMC $5\times25$.
Estimated training time per run averages 52 minutes for KP capacity (range 17--124 minutes), 6.7 hours for KP weight (2.4--12.5 hours), 2.8 hours for WSMC $10\times50$ (0.5--6.8 hours), and 62 minutes for WSMC $5\times25$ (12--219 minutes). These estimates combine time per epoch with the number of training epochs because some runs resumed across time-limited jobs. Inference requires one predictor evaluation and one first-stage solve. Gurobi solve times per test instance are 18.4~ms for KP capacity, 3.9~ms for KP weight, 13.9~ms for WSMC $10\times50$, and 7.8~ms for WSMC $5\times25$. The reported timings including model construction and evaluation range from 13 to 44~ms.
\subsubsection{Extended Results}
Tables~\ref{tab:app-ic-sckp} to~\ref{tab:app-ic-wsmc5} report regret, prediction error, and infeasibility for each problem and penalty setting. Table~\ref{tab:app-ic-tests} reports paired comparisons with SFGE and MSE across the 15 runs.
\begin{table}[ht]
\centering
\caption{Knapsack with stochastic capacity (KP capacity), Table~8 of \citet{silvestri2026score} over 15 runs. Bold underlined marks the lowest value in each block and column, and bold the second lowest.}
\label{tab:app-ic-sckp}
\vspace{4pt}
\begin{tabular}{lccc}
\toprule
Method & Rel.\ post-hoc regret & MSE & Infeas.\ ratio \\
\midrule
\multicolumn{4}{c}{50 items, $\rho = 5$} \\
\midrule
MSE & 0.5556 $\pm$ 0.3529 & \best{10.81} & \best{0.667 $\pm$ 0.257} \\
SFGE & \second{0.3577 $\pm$ 0.2152} & \second{16.63} & \second{0.749 $\pm$ 0.295} \\
PolyStepOR & \best{0.2975 $\pm$ 0.1666} & 21.87 & 0.752 $\pm$ 0.289 \\
\midrule
\multicolumn{4}{c}{50 items, $\rho = 10$} \\
\midrule
MSE & 1.2126 $\pm$ 0.7797 & \best{10.81} & \best{0.667 $\pm$ 0.257} \\
SFGE & \second{0.5081 $\pm$ 0.2784} & \second{23.58} & 0.907 $\pm$ 0.063 \\
PolyStepOR & \best{0.3802 $\pm$ 0.1958} & 32.04 & \second{0.847 $\pm$ 0.212} \\
\midrule
\multicolumn{4}{c}{50 items, $\rho = 20$} \\
\midrule
MSE & 2.5205 $\pm$ 1.6313 & \best{10.81} & \best{0.667 $\pm$ 0.257} \\
SFGE & \second{0.8216 $\pm$ 0.5031} & \second{32.66} & 0.955 $\pm$ 0.036 \\
PolyStepOR & \best{0.4772 $\pm$ 0.2415} & 37.68 & \second{0.877 $\pm$ 0.168} \\
\bottomrule
\end{tabular}
\end{table}

\begin{table}[ht]
\centering
\caption{Knapsack with stochastic weights (KP weight), Table~2 of \citet{silvestri2026score} over 15 runs. Bold underlined marks the lowest value in each block and column, and bold the second lowest.}
\label{tab:app-ic-swkp}
\vspace{4pt}
\begin{tabular}{lccc}
\toprule
Method & Rel.\ post-hoc regret & MSE & Infeas.\ ratio \\
\midrule
\multicolumn{4}{c}{50 items, $\rho = 5$} \\
\midrule
MSE & 0.1676 $\pm$ 0.0357 & \best{$7.88\times10^{4}$} & \second{0.929 $\pm$ 0.031} \\
CombOptNet & 0.1901 $\pm$ 0.0289 & $5.09\times10^{7}$ & \best{0.907 $\pm$ 0.032} \\
SFGE & \best{0.1297 $\pm$ 0.0114} & \second{$3.48\times10^{5}$} & 0.987 $\pm$ 0.017 \\
PolyStepOR & \second{0.1501 $\pm$ 0.0210} & $1.58\times10^{6}$ & 0.964 $\pm$ 0.021 \\
\midrule
\multicolumn{4}{c}{50 items, $\rho = 10$} \\
\midrule
MSE & 0.3190 $\pm$ 0.0814 & \best{$7.88\times10^{4}$} & \second{0.929 $\pm$ 0.031} \\
CombOptNet & 0.4015 $\pm$ 0.0659 & $5.09\times10^{7}$ & \best{0.907 $\pm$ 0.032} \\
SFGE & \best{0.1760 $\pm$ 0.0221} & \second{$3.66\times10^{5}$} & 0.989 $\pm$ 0.012 \\
PolyStepOR & \second{0.2293 $\pm$ 0.0372} & $2.01\times10^{6}$ & 0.980 $\pm$ 0.025 \\
\midrule
\multicolumn{4}{c}{50 items, $\rho = 20$} \\
\midrule
MSE & 0.6146 $\pm$ 0.1745 & \best{$7.88\times10^{4}$} & \second{0.929 $\pm$ 0.031} \\
CombOptNet & 0.8233 $\pm$ 0.1401 & $5.09\times10^{7}$ & \best{0.907 $\pm$ 0.032} \\
SFGE & \best{0.2205 $\pm$ 0.0276} & \second{$3.73\times10^{5}$} & 0.993 $\pm$ 0.009 \\
PolyStepOR & \second{0.2884 $\pm$ 0.0558} & $2.01\times10^{6}$ & 0.984 $\pm$ 0.024 \\
\bottomrule
\end{tabular}
\end{table}

\begin{table}[ht]
\centering
\caption{Weighted set multi-cover $10\times50$, Table~3 of \citet{silvestri2026score} over 15 runs. Bold underlined marks the lowest value in each block and column, and bold the second lowest.}
\label{tab:app-ic-wsmc10}
\vspace{4pt}
\begin{tabular}{lccc}
\toprule
Method & Rel.\ post-hoc regret & MSE & Infeas.\ ratio \\
\midrule
\multicolumn{4}{c}{$10\times50$, $\rho = 1$} \\
\midrule
MSE & 1.2960 $\pm$ 0.5676 & \best{$1.04\times10^{5}$} & 0.912 $\pm$ 0.048 \\
CombOptNet & 5.2044 $\pm$ 2.1887 & $1.89\times10^{7}$ & 1.000 $\pm$ 0.000 \\
SFGE & \second{1.1373 $\pm$ 0.4368} & \second{$3.33\times10^{5}$} & \best{0.809 $\pm$ 0.125} \\
PolyStepOR & \best{1.1205 $\pm$ 0.5002} & $4.87\times10^{5}$ & \second{0.839 $\pm$ 0.091} \\
\midrule
\multicolumn{4}{c}{$10\times50$, $\rho = 5$} \\
\midrule
MSE & 5.7115 $\pm$ 2.7809 & \best{$1.04\times10^{5}$} & 0.912 $\pm$ 0.048 \\
CombOptNet & 110.7887 $\pm$ 30.2424 & $1.89\times10^{7}$ & 1.000 $\pm$ 0.000 \\
SFGE & \second{2.5460 $\pm$ 0.9283} & \second{$6.43\times10^{5}$} & \best{0.340 $\pm$ 0.113} \\
PolyStepOR & \best{2.5253 $\pm$ 1.0119} & $8.88\times10^{5}$ & \second{0.452 $\pm$ 0.121} \\
\midrule
\multicolumn{4}{c}{$10\times50$, $\rho = 10$} \\
\midrule
MSE & 11.2307 $\pm$ 5.5511 & \best{$1.04\times10^{5}$} & 0.912 $\pm$ 0.048 \\
CombOptNet & 440.7395 $\pm$ 118.7445 & $1.89\times10^{7}$ & 1.000 $\pm$ 0.000 \\
SFGE & \best{3.3108 $\pm$ 1.1913} & \second{$7.80\times10^{5}$} & \best{0.235 $\pm$ 0.088} \\
PolyStepOR & \second{3.3701 $\pm$ 1.5660} & $1.08\times10^{6}$ & \second{0.345 $\pm$ 0.091} \\
\bottomrule
\end{tabular}
\end{table}

\begin{table}[ht]
\centering
\caption{Weighted set multi-cover $5\times25$, Table~9 of \citet{silvestri2026score} over 15 runs. Bold underlined marks the lowest value in each block and column, and bold the second lowest.}
\label{tab:app-ic-wsmc5}
\vspace{4pt}
\begin{tabular}{lrrr}
\toprule
Method & Rel.\ post-hoc regret & MSE & Infeas.\ ratio \\
\midrule
\multicolumn{4}{c}{$5\times25$, $\rho = 1$} \\
\midrule
MSE & 1.6856 $\pm$ 1.1119 & \best{$8.07\times10^{4}$} & 0.728 $\pm$ 0.163 \\
CombOptNet & 6.0568 $\pm$ 4.1781 & $8.66\times10^{6}$ & 1.000 $\pm$ 0.000 \\
SFGE & \second{1.2319 $\pm$ 0.5118} & \second{$2.98\times10^{5}$} & \best{0.533 $\pm$ 0.280} \\
PolyStepOR & \best{1.0940 $\pm$ 0.3827} & $7.20\times10^{5}$ & \second{0.568 $\pm$ 0.248} \\
\midrule
\multicolumn{4}{c}{$5\times25$, $\rho = 5$} \\
\midrule
MSE & 7.4364 $\pm$ 5.6404 & \best{$8.07\times10^{4}$} & 0.728 $\pm$ 0.163 \\
CombOptNet & 163.2305 $\pm$ 106.4945 & $8.66\times10^{6}$ & 1.000 $\pm$ 0.000 \\
SFGE & \second{2.3798 $\pm$ 0.7306} & \second{$5.28\times10^{5}$} & \best{0.181 $\pm$ 0.102} \\
PolyStepOR & \best{2.3524 $\pm$ 0.8307} & $7.38\times10^{5}$ & \second{0.279 $\pm$ 0.142} \\
\midrule
\multicolumn{4}{c}{$5\times25$, $\rho = 10$} \\
\midrule
MSE & 14.6249 $\pm$ 11.3160 & \best{$8.07\times10^{4}$} & 0.728 $\pm$ 0.163 \\
CombOptNet & 654.3986 $\pm$ 426.2526 & $8.66\times10^{6}$ & 1.000 $\pm$ 0.000 \\
SFGE & \best{2.8663 $\pm$ 0.8869} & \second{$6.33\times10^{5}$} & \best{0.120 $\pm$ 0.074} \\
PolyStepOR & \second{3.2120 $\pm$ 1.2114} & $7.18\times10^{5}$ & \second{0.184 $\pm$ 0.097} \\
\bottomrule
\end{tabular}
\end{table}

\begin{table}[ht]
\centering
\caption{Paired two-sided tests over the 15 paired runs of each cell (5 datasets $\times$ 3 splits) on the relative post-hoc regret. ``Lower mean'' names the method with the lower mean of the pair, and ``Runs'' counts the runs in which PolyStepOR has the lower regret.}
\label{tab:app-ic-tests}
\resizebox{0.95\textwidth}{!}{%
\begin{tabular}{lclccclc}
\toprule
& & \multicolumn{4}{c}{Against SFGE} & \multicolumn{2}{c}{Against MSE} \\
\cmidrule(lr){3-6}\cmidrule(lr){7-8}
Problem & $\rho$ & Lower mean & $p$ Wilcoxon & $p$ sign-flip & Runs & Lower mean & $p$ Wilcoxon \\
\midrule
KP (capacity) & 5 & PolyStepOR & 0.002 & $<0.001$ & 12/15 & PolyStepOR & 0.002 \\
 & 10 & PolyStepOR & $<0.001$ & $<0.001$ & 14/15 & PolyStepOR & $<0.001$ \\
 & 20 & PolyStepOR & $<0.001$ & $<0.001$ & 15/15 & PolyStepOR & $<0.001$ \\
\midrule
KP (weight) & 5 & SFGE & 0.004 & 0.005 & 2/15 & PolyStepOR & 0.022 \\
 & 10 & SFGE & $<0.001$ & $<0.001$ & 0/15 & PolyStepOR & $<0.001$ \\
 & 20 & SFGE & $<0.001$ & $<0.001$ & 1/15 & PolyStepOR & $<0.001$ \\
\midrule
WSMC $10\times50$ & 1 & PolyStepOR & 0.252 & 0.520 & 9/15 & PolyStepOR & 0.001 \\
 & 5 & PolyStepOR & 0.679 & 0.711 & 10/15 & PolyStepOR & $<0.001$ \\
 & 10 & SFGE & 0.561 & 0.880 & 9/15 & PolyStepOR & $<0.001$ \\
\midrule
WSMC $5\times25$ & 1 & PolyStepOR & 0.151 & 0.059 & 8/15 & PolyStepOR & 0.041 \\
 & 5 & PolyStepOR & 0.762 & 0.863 & 8/15 & PolyStepOR & $<0.001$ \\
 & 10 & SFGE & 0.048 & 0.162 & 3/15 & PolyStepOR & $<0.001$ \\
\bottomrule
\end{tabular}}
\end{table}

\paragraph{KP capacity.}
PolyStepOR achieves the lowest mean post-hoc regret at every penalty level, with a significant paired difference from SFGE in each setting. Increasing the predicted capacity enlarges the feasible set and can lead to selections that exceed the realized capacity. Recourse then removes items at $\rho$ times their value, whereas additions contribute only $1/\rho$ of their value. Larger penalties therefore increase the cost of removals relative to additions. At $\rho=10$ and $20$, PolyStepOR produces fewer infeasible first-stage decisions than SFGE, with ratios of $0.85$ and $0.88$ compared with $0.91$ and $0.96$. Its regret also grows less across the penalty settings, from $0.30$ to $0.48$, compared with $0.36$ to $0.82$ for SFGE. MSE has the lowest prediction error and infeasibility ratio but higher post-hoc regret, illustrating that these metrics measure different aspects of decision quality.

\paragraph{KP weight.}
SFGE achieves the lowest mean regret at every penalty level, while PolyStepOR improves over MSE and CombOptNet. First-stage infeasibility is high for all methods, ranging from approximately $0.90$ to $0.99$, so recourse plays a substantial role in evaluation. Unlike KP capacity, which predicts one scalar, this problem predicts 50 weights jointly. PolyStepOR searches 270 of the predictor's 300 parameters using one probe radius, compared with six parameters for KP capacity. Its prediction MSE is approximately 4.5 to 5.5 times that of SFGE. These differences characterize the more demanding setting but do not isolate the cause of the regret gap.

\paragraph{WSMC.}
On WSMC $10\times50$, we detect no significant difference in post-hoc regret between PolyStepOR and SFGE. Both outperform MSE by a larger margin as the penalty increases. Their first-stage infeasibility also decreases between $\rho=1$ and $\rho=10$, from $0.84$ to $0.35$ for PolyStepOR and from $0.81$ to $0.24$ for SFGE, while MSE remains at $0.91$. This is consistent with stronger penalties encouraging decisions that satisfy more of the realized coverage requirements.

The smaller WSMC $5\times25$ problem shows a similar pattern. At $\rho=10$, SFGE has lower mean regret, with a Wilcoxon $p$-value of $0.048$ but a sign-flip $p$-value of $0.162$. Evidence for a difference in this setting therefore depends on the test used.

\subsection{Extended Details on Network Districting}\label{app:districting}

\subsubsection{Problem Setup}
We follow the network districting problem of \citet{ahmed2024districtnet}. A city is represented as a graph whose vertices are basic units (BUs) with stochastic customer demand, together with a depot. The $N$ BUs are partitioned into $k=\lfloor N/t\rfloor$ connected districts, each constrained to contain within $\pm20\%$ of the target size $t$. For a district $d$, the routing cost is the expected length of the shortest depot tour serving its realized demand. Using the authors' evaluator with $|\Omega|=100$ demand scenarios, the cost of a districting $\lambda$ is
\begin{equation}\label{eq:app-dn-cost}
\mathrm{Cost}(\lambda)=\sum_{d\in\lambda}\frac{1}{|\Omega|}\sum_{\omega\in\Omega}\mathrm{TSP}(d,\bm{\xi}^{\omega}).
\end{equation}

DistrictNet does not optimize this routing objective directly. Instead, a graph neural network predicts edge weights for a capacitated minimum spanning tree (CMST) surrogate, whose solution induces the district partition. At test time, the surrogate is solved using the authors' iterated local search, with a time limit of 20 minutes for instances below 400 BUs and 60 minutes otherwise. PolyStepOR retains the same predictor, surrogate problem, solver, and downstream evaluator, but trains the predictor directly from the districting costs in Eq.~\eqref{eq:app-dn-cost}. :chatgpt-content-reference{index="0"}

\paragraph{Data.}
Training uses the 100 cities of \citet{ahmed2024districtnet}, each containing $N=30$ BUs with target size $t=3$ and generated from 27 English cities outside the test set. DistrictNet is trained from optimal districtings of these instances, obtained by enumerating feasible districts and solving the resulting optimization problem. Generating such labels becomes rapidly expensive: \citet{ahmed2024districtnet} report approximately 400 CPU-core days for one instance with 60 BUs and 10 districts. The test set contains the 35 instances of their Table~1, corresponding to seven cities (Bristol, Leeds, London, Lyon, Manchester, Marseille, and Paris) with $N=120$ and $t\in\{3,6,12,20,30\}$. We additionally evaluate on the Ile-de-France instance with 2000 BUs and $t=20$. Both test settings are therefore substantially larger than the 30-BU training instances. :chatgpt-content-reference{index="1"}

\paragraph{Evaluation.}
Each returned districting is evaluated by Eq.~\eqref{eq:app-dn-cost} using the same 100 demand scenarios for a given problem. For method $m$, we report its mean relative cost to PolyStepOR over the 35 test problems,
\begin{equation}
    \frac{100}{35}\sum_{p=1}^{35}
    \frac{\mathrm{Cost}_{m}(p)-\mathrm{Cost}_{\rm PS}(p)}
    {\mathrm{Cost}_{\rm PS}(p)},
\end{equation}
together with a paired two-sided Wilcoxon test over the 35 instances. \citet{ahmed2024districtnet} use the analogous relative-cost metric with DistrictNet as the reference and a one-sided test.

\subsubsection{Experimental Setup}

\paragraph{PolyStepOR.}
PolyStepOR trains the authors' graph network, containing $25{,}857$ parameters, on the same 100 training cities. For each candidate predictor, the predicted edge weights define a CMST problem, and the resulting districting is evaluated using the authors' precomputed district costs. The training objective is the mean districting cost over all 100 cities. Neither the optimal training districtings nor pretrained DistrictNet parameters are used, and each run starts from a random initialization determined by its seed. All remaining components, including the network architecture, data, demand scenarios, optimization model, and downstream evaluator, are inherited from \citet{ahmed2024districtnet}.

We report both an untuned and a tuned configuration, with their hyperparameters listed in Table~\ref{tab:repro}. The tuned configuration is selected using the second set of 100 small cities provided by \citet{ahmed2024districtnet}, which serves as an in-distribution validation set. Within each run, the returned predictor is the probed candidate with the lowest training cost. For one untuned run, a single training solve did not terminate; this run therefore uses a finite solve limit, retaining the incumbent when available and assigning a fixed failure cost otherwise.

\paragraph{Baselines.}
DistrictNet is retrained with the authors' code at each of the ten PolyStepOR seeds and evaluated on the same demand scenarios. The results for BD, FIG, PredGNN, and AvgTSP are taken directly from \citet{ahmed2024districtnet}. For Ile-de-France, all baseline values, including DistrictNet, are the published results of \citet{ahmed2024districtnet}.

\paragraph{Compute.}
For the untuned configuration, each PolyStepOR training step evaluates 1024 candidate predictors. A step takes 15.9 minutes on 64 CPU cores on average (range: 10.7--29.3 minutes) and 38.3 minutes on 16 cores, giving approximately 8 hours for a 30-step run on 64 cores. Seeds 17, 42, and 71 were trained on 16 cores and took 16.8--20.9 hours; the remaining seeds took 8.3--10.9 hours. Candidate evaluations are parallelized across worker processes, each holding a copy of the training data, resulting in a peak memory usage of approximately 55~GB on 64 cores.

The tuned configuration evaluates 4096 candidates per step and requires 27--31 minutes per step on 128 cores, corresponding to 13.5--14.5 hours per run. Retraining DistrictNet takes 32.5 minutes on 2 cores. At inference time, both methods use the authors' iterated local search with the same time limits. Runtime is approximately 21.4 minutes per $N=120$ instance (20.9--23.3 minutes) and 64.2 minutes on Ile-de-France. Evaluating the returned districtings over the demand scenarios requires an additional 10.7 minutes per $N=120$ instance on average, and 7--14 minutes per method on Ile-de-France.

\subsubsection{Extended Results}
Table~\ref{tab:app-dn-seeds} reports the tuned PolyStepOR configuration for each training seed together with DistrictNet retrained at the same seed. Table~\ref{tab:app-dn-configs} aggregates the tuned and untuned configurations over their respective seeds. On the 35 benchmark problems, the tuned configuration and DistrictNet have similar costs, with no significant difference under the paired Wilcoxon test ($p=0.14$). The untuned configuration performs worse, with DistrictNet reducing cost by $2.0\%$ on average ($p=2\times10^{-10}$). On Ile-de-France, the tuned PolyStepOR runs achieve a mean cost of $2197.0$, with the best run reaching $2156.3$, compared with the published DistrictNet cost of $2205.7$. Figure~\ref{fig:app-dn-maps} visualizes the districtings produced by both methods on the examples considered by \citet{ahmed2024districtnet}. :chatgpt-content-reference{index="2"}

\begin{table}[ht]
\centering
\setlength{\tabcolsep}{4pt}
\caption{Districting results per training seed for tuned PolyStepOR and DistrictNet retrained with the authors' code at the same seed. Results are reported over the 35 benchmark problems together with the PolyStepOR cost on Ile-de-France. Relative cost is measured for DistrictNet with respect to PolyStepOR, so negative values indicate lower DistrictNet cost. The $p$-value is from a paired two-sided Wilcoxon test over the 35 problems.}
\label{tab:app-dn-seeds}
\vspace{4pt}
\resizebox{0.8\linewidth}{!}{%
\begin{tabular}{rrrrrrr}
\toprule
& \multicolumn{5}{c}{35 problems} & Ile-de-France \\
\cmidrule(lr){2-6}\cmidrule(lr){7-7}
Seed & PolyStepOR & DistrictNet & Rel.\ (\%) & $p$ & PolyStepOR cheaper & PolyStepOR \\
\midrule
3 & 551.7 & 551.4 & -0.71 & 0.53 & 15/35 & 2156.3 \\
5 & 552.2 & 562.0 & +0.79 & $6\times10^{-3}$ & 22/35 & 2268.3 \\
7 & 555.6 & 557.8 & -0.16 & 0.64 & 19/35 & 2190.7 \\
11 & 554.5 & 556.0 & -0.97 & 0.40 & 13/35 & 2219.5 \\
13 & 554.4 & 557.2 & +0.54 & 0.08 & 25/35 & 2183.0 \\
17 & 551.5 & 556.1 & +0.50 & 0.34 & 20/35 & 2200.6 \\
19 & 554.4 & 557.4 & -0.89 & 0.70 & 18/35 & 2191.6 \\
42 & 557.3 & 561.0 & -0.64 & 0.39 & 16/35 & 2203.9 \\
71 & 552.4 & 558.6 & +0.90 & 0.14 & 20/35 & 2162.7 \\
1234 & 556.3 & 559.5 & +0.55 & 0.20 & 21/35 & 2193.7 \\
\midrule
Mean over seeds & 554.0 & 557.7 & $-0.05$ & & & 2197.0 \\
\bottomrule
\end{tabular}}
\end{table}

\begin{table}[ht]
\centering
\caption{Performance of tuned and untuned PolyStepOR configurations. Results on the 35 benchmark problems are compared with DistrictNet retrained at the same seeds. Ile-de-France reports the mean and best PolyStepOR cost across seeds.}
\label{tab:app-dn-configs}
\vspace{4pt}
\resizebox{0.95\textwidth}{!}{%
\begin{tabular}{lccccccc}
\toprule
& & \multicolumn{4}{c}{35 problems} & \multicolumn{2}{c}{Ile-de-France} \\
\cmidrule(lr){3-6}\cmidrule(lr){7-8}
Configuration & Seeds & PolyStepOR & Rel.\ DistrictNet (\%) & $p$ & PolyStepOR cheaper & Mean & Best \\
\midrule
Tuned  & 10 & 554.0 & -0.05 & 0.14 & 18/35 & 2197.0 & 2156.3 \\
Untuned & 10 & 565.3 & -1.97 & $2\times10^{-10}$ & 1/35 & 2220.5 & 2186.6 \\
\bottomrule
\end{tabular}}
\end{table}

\begin{figure}[p]
\centering
\includegraphics[height=0.8\textheight,keepaspectratio]{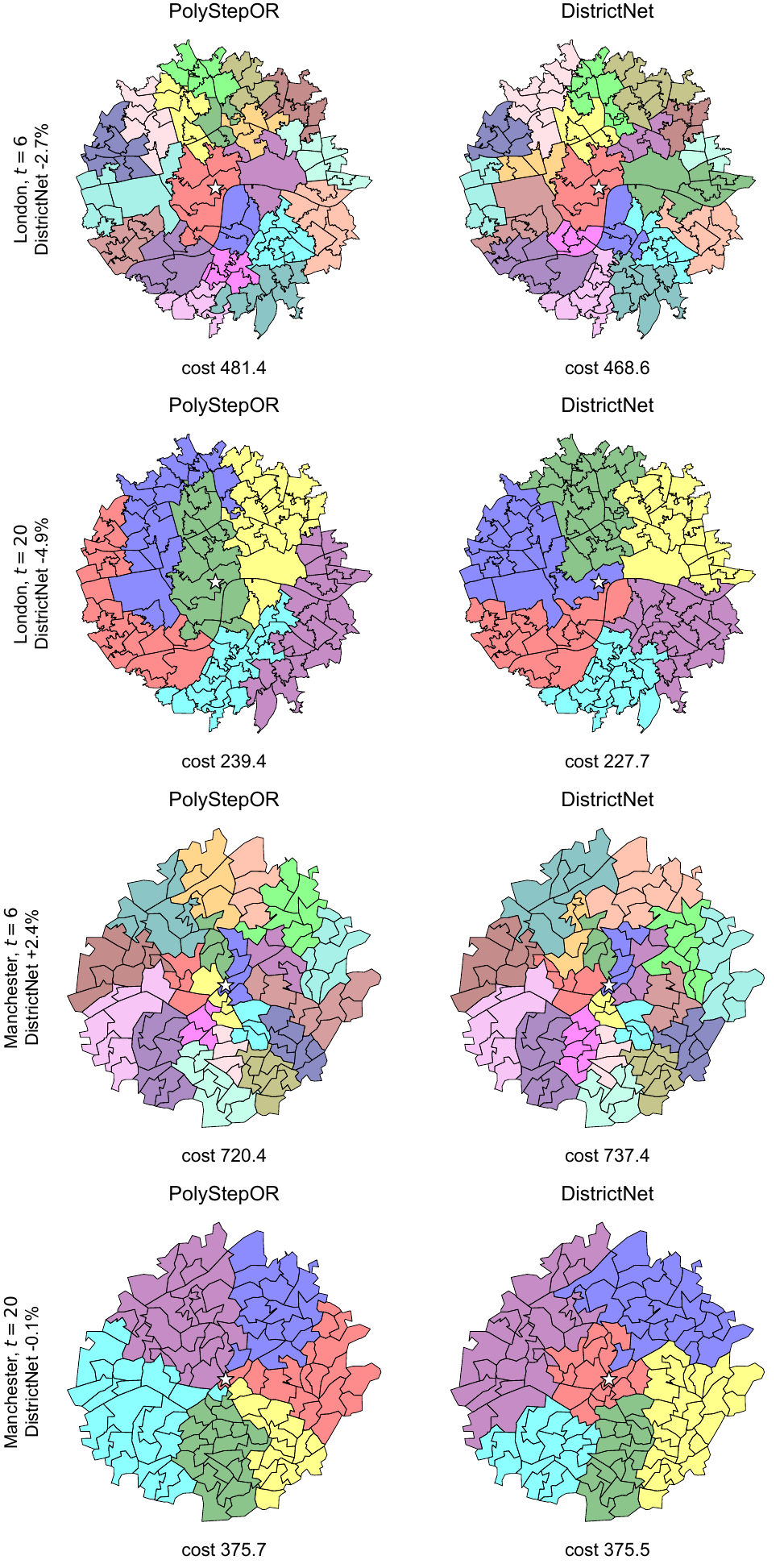}
\caption{Districtings produced by PolyStepOR and DistrictNet at seed 1234 on the examples visualized by \citet{ahmed2024districtnet}: London and Manchester at $t=6$ and $t=20$. The cost is shown below each map, and the row label reports the relative cost of DistrictNet with respect to PolyStepOR.}
\label{fig:app-dn-maps}
\end{figure}

\subsection{Extended Details on Brass Alloy Allocation}\label{app:alloy}
\subsubsection{Problem Setups}

We consider the brass alloy production problem of \citet{mandi2026feasibility}, adapted from \citet{hu2023two}. A factory purchases ore from $N=10$ suppliers to meet copper (Cu) and zinc (Zn) requirements at minimum cost. Supplier $n$ charges $q_n$ per unit of ore, which contains $a_{nm}$ units of metal $m\in\{\mathrm{Cu},\mathrm{Zn}\}$. The metal contents are unknown at decision time and must be predicted. Given these contents, the purchase problem is modeled as:
\begin{equation*}
\min_{\bm{x}}\ \sum_{n=1}^{N}q_nx_n
\quad\text{s.t.}\quad
\sum_{n=1}^{N}a_{nm}x_n\geq b_m
\quad\forall m\in\{\mathrm{Cu},\mathrm{Zn}\},
\qquad
\bm{x}\geq\bm{0},
\end{equation*}
where $x_n$ is the quantity purchased from supplier $n$. The requirements are $b_{\mathrm{Cu}}=627.54$ and $b_{\mathrm{Zn}}=369.72$. At decision time, predicted contents replace $a_{nm}$ in the constraints. The dataset contains 500 instances, split into 350 training, 50 validation, and 100 test instances. Each metal content is predicted from 4,096 features using a fully connected network with one hidden layer of 512 units. Although \citet{mandi2026feasibility} describe a linear program, their implementation restricts the purchase quantities to integers. We follow that implementation and add $\bm{x}\in\mathbb{Z}_{\geq0}^{N}$ to the formulation above.

\paragraph{Metrics.}
Let $\mathcal{I}$ denote the test instances whose predicted purchases satisfy the requirements under the true metal contents, and let $r_i$ be the optimal purchase cost for instance $i$. Following Eq.~23 of \citet{mandi2026feasibility}, we report
\begin{equation}\label{eq:app-alloy-metrics}
\operatorname{Infeas}
=
1-\frac{|\mathcal{I}|}{N_{\rm test}},
\qquad
\operatorname{CRegret}
=
\frac{1}{|\mathcal{I}|}
\sum_{i\in\mathcal{I}}
\frac{\bm{q}_i^\top\hat{\bm{x}}_i-r_i}{r_i}.
\end{equation}
Infeasibility is measured before correction. Conditional regret is computed only on each method's feasible purchases and is undefined when $\mathcal{I}$ is empty. Consequently, a low conditional regret can reflect performance on a small or easier subset of instances rather than good performance across the full test set.

\subsubsection{Experimental Setups}

\paragraph{PolyStepOR.}
We train the authors' network from random initialization using their training loop. Feasible purchases are evaluated by their purchase cost. For an infeasible purchase, the authors' correction scales the quantities until both metal requirements are satisfied. Let $\bm{a}_{i,m}\in\mathbb{R}^{N}$ collect the true contents of metal $m$ across suppliers in instance $i$. The correction factor is
\begin{equation*}
\tau_i
=
\max\left\{
1,\,
\max_{m\in\{\mathrm{Cu},\mathrm{Zn}\}}
\frac{b_m}{\bm{a}_{i,m}^{\top}\hat{\bm{x}}_i}
\right\},
\end{equation*}
provided the denominators are positive. As in the authors' evaluator, the corrected purchase $\tau_i\hat{\bm{x}}_i$ is not rounded, so it can be fractional. The denominators are positive for every instance, since every true content in the data is at least 15 and a returned purchase is nonzero because the requirements are positive. Writing $\bm{a}_i$ for the collection of true metal contents, the training loss is:
\begin{equation}\label{eq:app-alloy-loss}
\ell(\hat{\bm{x}}_i,\bm{a}_i)
=
\bm{q}_i^\top(\tau_i\hat{\bm{x}}_i)
+
\mu_i(\tau_i-1)\bm{q}_i^\top\hat{\bm{x}}_i,
\end{equation}
where $\mu_i$ is the instance-specific penalty coefficient supplied with the data. The first term is the corrected purchase cost, and the second penalizes the additional purchase. For an initially feasible decision, $\tau_i=1$ and the loss reduces to its original cost. If the solver returns no purchase, the assigned loss is $100\,\mu_i\|\bm{q}_i\|_1$. These evaluations require no optimal reference purchase. Unlike the other experiments, the configuration is chosen by validation cost among the points of a 32-point scrambled Sobol design trained on separate tuning seeds, and then retrained on the evaluation seeds (Table~\ref{tab:repro}). The search subspace is rebuilt at each step from previous displacements and random directions, rather than using the fixed basis of Section~\ref{sec:polystep}, and the number of probe radii drops from 2 to 1 after three consecutive steps with decreasing cost. These practical variants differ from the fixed-setting algorithm considered in the stationarity analysis. Training follows the authors' setup with shuffling between epochs, and we evaluate the final-epoch model, solving every candidate with the authors' Gurobi model in a single solver process.

\paragraph{Baselines.}
For infeasibility and conditional regret, we use the results reported by \citet{mandi2026feasibility} for MSE, CombOptNet, SFL~\citep{nandwani2022solverfree}, 2sPtO~\citep{hu2023two}, and ODECE with $\alpha\in\{0.2,\ldots,0.8\}$. We additionally run SFGE using the authors' estimator and published post-hoc regret objective. Three of its five seeds produce no feasible test purchase, so conditional regret is averaged over the remaining two seeds. For the full-test post-hoc regret comparison, we separately rerun MSE and ODECE using the authors' code and the same five evaluation seeds as PolyStepOR. These reruns are distinct from the published baseline values in Table~\ref{tab:re-alloy}.
\paragraph{Feasibility.} Following \citet{mandi2026feasibility}, the purchase $\hat{\bm{x}}_i$ predicted for test instance $i$ is feasible when it meets both requirements under the true contents,
\begin{equation}\label{eq:app-alloy-feasible}
\hat{\bm{x}}_i\ \text{is feasible}\iff\bm{a}_{i,m}^{\top}\hat{\bm{x}}_i\geq b_m\quad\forall m\in\{\mathrm{Cu},\mathrm{Zn}\}.
\end{equation}
The set $\mathcal{I}$ in Eq.~\eqref{eq:app-alloy-metrics} collects the test instances that satisfy Eq.~\eqref{eq:app-alloy-feasible}, and the check is made on the purchase before any correction.

\paragraph{Computational cost.} A PolyStepOR training run takes 2.9 hours on average (2.8 to 3.3 over the five seeds), using one solver process. The solver calls take about $27\%$ of the time. The rest is dominated by the forward passes of the network over all probe candidates, which run on the CPU alongside the solver. Evaluating the candidates as one batch on a GPU, as the PolyStep library supports, is much cheaper. In our GPU-batched runs of the same network, the forward passes took about $11\%$ of the training time with 24 times more candidates per step.

\subsubsection{Extended Results}
Table~\ref{tab:re-alloy} reports the values shown in Figure~\ref{fig:alloy-tradeoff}. The PolyStepOR row corresponds to the moving mode. Its first-stage infeasibility is $0.555\pm0.013$, compared with $0.532$ for MSE, while its conditional regret is $0.154\pm0.005$, compared with $0.169$. Thus, PolyStepOR has slightly lower reported conditional regret but a higher infeasibility rate. The standard deviations describe variability across its five evaluation seeds. ODECE varies the emphasis on feasibility through $\alpha$. At $\alpha=0.4$, its infeasibility rate of $0.571$ is close to PolyStepOR's, with conditional regret of $0.195$. Increasing $\alpha$ reduces infeasibility but raises conditional regret, reaching $0.174$ and $0.486$, respectively, at $\alpha=0.8$. PolyStepOR uses the dataset's prescribed recourse penalties rather than an analogous sweep over $\alpha$, so its reported configuration provides one operating point.

SFL and CombOptNet are infeasible on $99.1\%$ and $89.6\%$ of test instances, respectively, leaving small subsets on which to compute conditional regret. The infeasibility rate of 2sPtO is lower at $40.0\%$, but its conditional regret is $8.539$. SFGE has an infeasibility rate of $73.9\%$ and conditional regret of $0.268$, with the latter defined for only two of its five evaluation seeds. Because conditional regret uses different feasible subsets, these values alone do not establish an overall ranking of decision quality. We therefore also evaluate post-hoc regret across the full test set, including correction and penalty costs. PolyStepOR achieves $0.231\pm0.005$, compared with $0.243\pm0.007$ for our MSE reruns. ODECE ranges from $0.230$ at $\alpha=0.3$ to $0.543$ at $\alpha=0.8$. This comparison evaluates all methods on the same instances and complements the separate assessment of feasibility before correction.

\begin{table}[ht]
\centering
\caption{Brass alloy production, with infeasibility before repair and regret over each method's feasible test purchases; lower is better. The baselines except SFGE are the values reported by \citet{mandi2026feasibility}, and PolyStepOR is mean $\pm$ standard deviation over five seeds.}
\label{tab:re-alloy}
\begin{tabular}{lcc}
\toprule
Method & Infeasibility & Conditional regret \\
\midrule
MSE & 0.532 & 0.169 \\
CombOptNet & 0.896 & 0.181 \\
SFL & 0.991 & 0.019 \\
2sPtO & 0.400 & 8.539 \\
ODECE ($\alpha=0.2$) & 0.790 & 0.081 \\
ODECE ($\alpha=0.3$) & 0.729 & 0.123 \\
ODECE ($\alpha=0.4$) & 0.571 & 0.195 \\
ODECE ($\alpha=0.5$) & 0.471 & 0.270 \\
ODECE ($\alpha=0.6$) & 0.324 & 0.365 \\
ODECE ($\alpha=0.7$) & 0.276 & 0.406 \\
ODECE ($\alpha=0.8$) & 0.174 & 0.486 \\
SFGE & 0.739 & 0.268 \\
\midrule
PolyStepOR & 0.555 $\pm$ 0.013 & 0.154 $\pm$ 0.005 \\
\bottomrule
\end{tabular}
\end{table}

\subsection{Reproducibility Configurations}\label{app:reproducability}
We provide the hyperparameter configurations, as well as the seeds used in each experiment to better guide the reproducibility of our algorithm, which are presented in Table~\ref{tab:repro}.
\begin{table}[!ht]
\centering
\caption{PolyStepOR settings behind every reported number.}
\label{tab:repro}
\resizebox{0.95\textwidth}{!}{%
\begin{tabular}{lccccccccccccc}
\toprule
\multirow{2}{*}{Problem} & \multirow{2}{*}{Schedule} & \multicolumn{2}{c}{$\varepsilon$} & \multicolumn{2}{c}{$s$} & \multicolumn{2}{c}{$r$} & \multirow{2}{*}{$d$} & \multirow{2}{*}{$q$} & \multirow{2}{*}{$K$} & \multirow{2}{*}{Momentum} & \multirow{2}{*}{$B$} & \multirow{2}{*}{Budget (steps)} \\
\cmidrule(lr){3-4}\cmidrule(lr){5-6}\cmidrule(lr){7-8}
 & & start & end & start & end & start & end & & & & & & \\
\midrule
\multicolumn{14}{c}{\textit{Classical OR}} \\
\midrule
Shortest path & cosine & 10 & 0.1 & 5 & 1 & 10 & 2 & 240 & 8 & 1 & 0 & 128 & 150 \\
Portfolio & cosine & 5 & 0.3 & 32 & 8 & 2 & 0.5 & 300 & 8 & 1 & 0 & 128 & 100 \\
Knapsack & cosine & 10 & 0.1 & 5 & 1 & 10 & 2 & 9 & 8 & 1 & 0 & 128 & 100 \\
Energy scheduling & cosine & 10 & 0.1 & 5 & 1 & 10 & 2 & 9 & 8 & 1 & 0 & 128 & 100 \\
Matching & cosine & 10 & 0.1 & 5 & 1 & 10 & 2 & 512 & 8 & 1 & 0 & 17 & 100 \\
\midrule
\multicolumn{14}{c}{\textit{In-constraint Prediction}} \\
\midrule
KP capacity & constant & 0.5 & 0.5 & 1.0 & 1.0 & 0.5 & 0.5 & 6 & 8 & 1 & 0 & 32 & 11{,}500 \\
KP weight & constant & 0.5 & 0.5 & 1.0 & 1.0 & 0.5 & 0.5 & 270 & 8 & 1 & 0 & 32 & 11{,}500 \\
WSMC $5\times25$ & constant & 0.5 & 0.5 & 1.0 & 1.0 & 0.5 & 0.5 & 30 & 8 & 1 & 0 & 32 & 11{,}500 \\
WSMC $10\times50$ & constant & 0.5 & 0.5 & 1.0 & 1.0 & 0.5 & 0.5 & 60 & 8 & 1 & 0 & 32 & 11{,}500 \\
\midrule
\multicolumn{14}{c}{\textit{Districting}} \\
\midrule
Tuned & linear & 0.174 & 0.174 & 86.4 & 19.8 & 18.3 & 18.3 & 512 & 512 & 4 & 0.9 & 100 & 30 \\
Untuned & cosine & 10 & 0.1 & 5 & 1 & 10 & 2 & 512 & 512 & 1 & 0 & 100 & 30 \\
\midrule
\multicolumn{14}{c}{\textit{Brass Alloy Production}} \\
\midrule
Brass alloy & linear & 0.707 & 0.707 & 2.67 & 1.26 & 461 & 461 & 128 & 8 & 2 & 0 & 32 & 220 \\
\bottomrule
\end{tabular}}
\end{table}
The following seeds are used for the reported experiments:
\begin{itemize}
\item Classical OR: Seeds $\in [0, 9]$
\item In-constraint Predictions: data-set seeds 0 to 4 times split seeds 0 to 2 (15 runs)
\item Districting: Seeds $\in 3, 5, 7, 11, 13, 17, 19, 42, 71, and 1234$.
\item Brass Alloy Production: Seeds $\in [11,15]$
\end{itemize}

\end{document}